\documentclass[10pt,journal,compsoc]{IEEEtran}
\usepackage{multirow}
\usepackage{booktabs,threeparttable}
\usepackage{epsfig}
\usepackage{arydshln}
\usepackage{epstopdf}
\usepackage{subfigure}
\usepackage{pseudocode}
\usepackage{rotating}
\usepackage{algpseudocode}
\usepackage{amsmath}
\usepackage{xcolor}
\usepackage{algorithm}
\usepackage{mdwlist}
\usepackage{algorithmicx}
\usepackage{bm}
\usepackage{amsthm}
\usepackage{soul}
\usepackage{multirow}
\usepackage{caption}
\usepackage{color, colortbl}
\usepackage[colorlinks,
            linkcolor=black,
            anchorcolor=black,
            urlcolor=black,
            citecolor=black
            ]{hyperref}
\usepackage{textcomp}
\usepackage{graphicx}
\usepackage{xr}
\usepackage{enumerate}
\def\BibTeX{{\rm B\kern-.05em{\sc i\kern-.025em b}\kern-.08em
    T\kern-.1667em\lower.7ex\hbox{E}\kern-.125emX}}
\newcommand{\vx}{\bm{x}}
\newcommand{\vX}{\bm{X}}

\newcommand{\vL}{\bm{L}}

\newcommand{\vD}{\bm{D}}
\newcommand{\vS}{\bm{S}}
\newcommand{\vv}{\bm{v}}
\newcommand{\vV}{\bm{V}}
\newcommand{\vT}{\bm{T}}
\newcommand{\vA}{\bm{A}}
\newcommand{\vB}{\bm{B}}
\newcommand{\cT}{\mathcal{T}}
\newtheorem{theorem}{Theorem}[section]

\usepackage{amsmath}
\usepackage{amssymb}
\usepackage{empheq}
\usepackage{xcolor}
\usepackage{graphicx}
\definecolor{darkyellow}{rgb}{1,0.9,0}
\definecolor{Gray}{gray}{0.95}
\definecolor{LightCyan}{rgb}{0.88,1,1}

\usepackage{amsfonts}

\ifCLASSOPTIONcompsoc
  \usepackage[nocompress]{cite}
\else
  \usepackage{cite}
\fi

\ifCLASSINFOpdf
\else
\fi

\begin{document}

\title{Integrating Tensor Similarity  to Enhance Clustering Performance}

\author{Hong~Peng,
Yu~Hu,
        Jiazhou~Chen,
        Haiyan~Wang,
        Yang~Li,
        and~Hongmin~Cai% <-this % stops a space
\IEEEcompsocitemizethanks{\IEEEcompsocthanksitem The authors are with the Department
of Computer Science and Engineering, South China University of Technology, Guangdong,
China.\protect\\
E-mail: hmcai@scut.edu.cn
 }% <-this % stops an unwanted space
\thanks{}}

\markboth{IEEE Transactions on Pattern Analysis and Machine Intelligence}%
{Cai \MakeLowercase{\textit{et al.}}}

\IEEEtitleabstractindextext{%
\begin{abstract}
The performance of most the clustering methods hinges on the used pairwise affinity, which is usually denoted by a similarity matrix. However, the pairwise similarity is notoriously known for its vulnerability of noise contamination or the imbalance in samples or features, and thus hinders accurate clustering. To tackle this issue, we propose to use information among samples to boost the clustering performance. We proved that a simplified  similarity for pairs, denoted by a fourth order tensor, equals to the Kronecker product of pairwise similarity matrices under decomposable assumption, or provide complementary information for which the pairwise similarity missed under indecomposable assumption. Then a high order similarity matrix is obtained from the tensor similarity via eigenvalue decomposition. The high order similarity capturing spatial information serves as a robust complement for the pairwise similarity. It is further integrated with the popular pairwise similarity, named by IPS2,  to boost the clustering performance. Extensive experiments demonstrated that the proposed IPS2 significantly outperformed previous similarity-based methods on real-world datasets and it was capable of handling the clustering task over under-sampled and noisy datasets.

\end{abstract}

}

\maketitle

\IEEEdisplaynontitleabstractindextext

\IEEEpeerreviewmaketitle

\IEEEraisesectionheading{
\section{Introduction}\label{sec:introduction}}

Clustering refers to separating observed data into groups with some quantified measurements, such that objects are similar to those in the same subgroup~\cite{liu2018global,abbe2018community}. Therefore, the similarity used to measure the affinity among samples plays a crucial role to reveal the neighborhood structure of samples and further facilitates clustering algorithms~\cite{yang2004similarity,lei2015consistency}. A popular way to model a pairwise relationship of samples is by utilizing the predefined or adaptively learned measurements. For example, the $k$-means and its variants use the Gaussian kernel function to compute the pairwise similarity~\cite{kanungo2002efficient,huang2005automated}. However, the pairwise similarity induced by distance metrics is vulnerable to outliers or noise corruptions which limits its usage in practice.

Several studies propose to extract high order information rather than the popular pairwise similarity to overcome its vulnerability and thereby enhance its clustering performance\cite{wang2018searching,nguyen2020learning}. For example, \cite{li2017inhomogeneous,zhou2007learning} approximate local manifolds among multiple samples to define a high order similarity. \cite{agarwal2006higher} defines a high order similarity by introducing a hypergraph. The hypergraph leverages the local structures for multiple samples to construct a highly expressive similarity~\cite{purkait2017clustering}. Although the hypergraph clustering algorithms take the connections among multiple samples into consideration, the learned similarity mainly relies on a pairwise relationship but also inherits its drawbacks.

Early attempts show that directly utilizing high order or pairwise information seems to be problematic. To tackle these issues, we propose to incorporate the high order relations among multiple samples with the popular pairwise similarity to boost the classical clustering performance. The high order similarity among multiples is denoted by a tensor, as the pairwise similarity does by a matrix. The tensor similarity aims to provide complementary information which a pairwise similarity fails to capture. To model the high order relations among samples, we take a $4$-th order tensor as an example. As depicting the sample-to-sample relation in a matrix, we characterize the pairs-to-pairs one in such a 4-th order tensor, so that the high order relations can be naturally extracted from it. We define two types of tensor similarities, the decomposable and indecomposable ones. From the indecomposable tensor similarity, a high order pairwise similarity is derived by eigenvalue decomposition. In this way, we build a bridge relating the sample-to-sample and pair-to-pair similarities for exploring the high order information. Finally, the two types of similarities are aggregated to achieve accurate clustering performance. We conducted extensive experiments to demonstrate that incorporating the tensor high order similarity can enhance the performance of popular similarity-based methods. Thus, the proposed method achieves superior performance comparing to state-of-art algorithms.

\begin{figure*}[!htb]
   \centering
        {\includegraphics[width = 7.2 in]{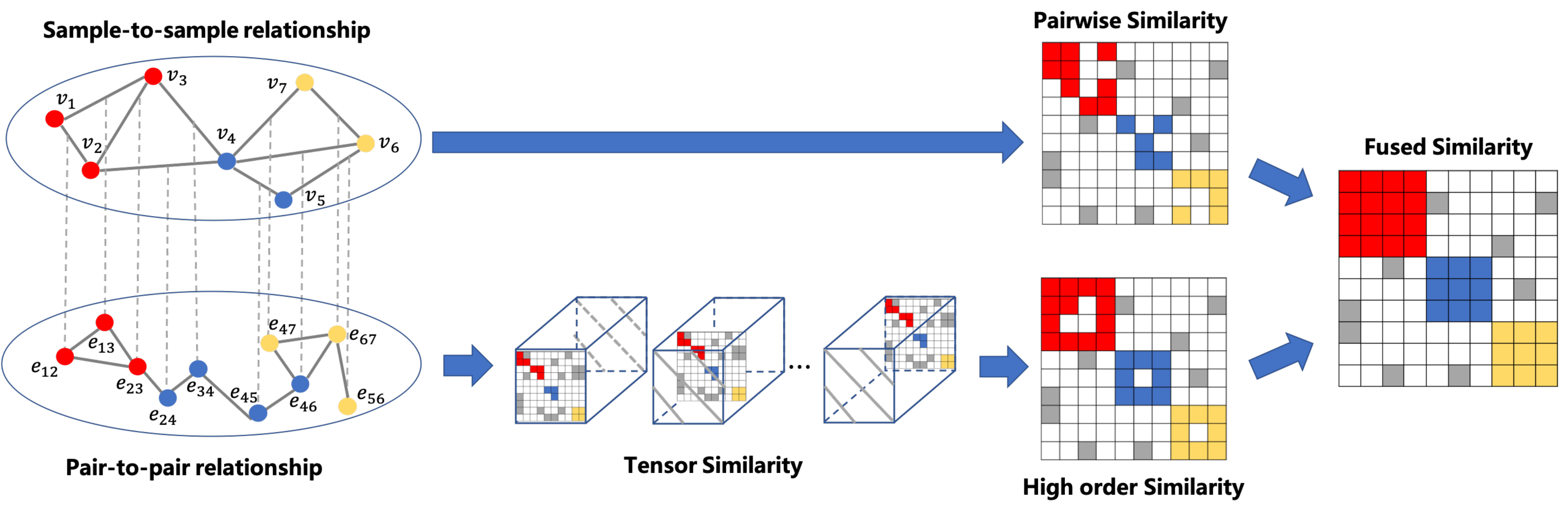}}
   \caption{The framework of the IPS2 algorithm. Given a dataset as input, IPS2 first learns both pairwise relationships and information among pairs. Then it proposes to encode pair-to-pair relationships into a tensor similarity for exploring the spatial structure of data. Based on the tensor similarity, IPS2 extracts a high order similarity, which can provide information ignored by the pairwise similarity. Afterwards, IPS2 combines the  pairwise and the high order relationships to generate the fused similarity for achieving accurate and robust clustering.}\label{fig:framework}
\end{figure*}

This paper achieves three contributions summarized as follows:
\begin{enumerate}
\item We develop a new clustering algorithm, named IPS2, by incorporating a high order similarity of data with a pairwise similarity. The high order similarity captures structural information of data which can be obtained from relationships between pairs, and it serves as complementary information for pairwise one.
\item We construct two types of tensor similarities to encode pairs' relationships and relate them to pairwise similarity. These enable us to establish a bridge between the learned high order similarity and pairwise one.
\item Extensive experiments validate the proposed IPS2 outperforms other baseline algorithms and has an advantage in the task of clustering on under-sampled and noisy datasets.
\end{enumerate}

The remainder of the paper is structured as follows: In Section~\ref{sec:related work}, we review the previous related works. Section~\ref{sec:method} describes our proposed method through formulating tensor similarity for learning a high order affinity of data. Section~\ref{Results} demonstrates the effectiveness of our approach through extensive experiments performed on synthetic and real-world datasets. Finally, we draw conclusions in Section~\ref{conclusion}.

\section{RELATED WORK}\label{sec:related work}

The performance of the classical similarity-based clustering methods depends on the accurate pairwise relationships. It is common to express the sample-to-sample similarity as a predefined distance function, and different metrics study distinct relationships\cite{mori2015similarity}. The popular spectral clustering uses a similarity computed by the Gaussian distance. Other well-known metrics, such as cosine distance and Manhattan distance, are used in different applications~\cite{Tao2017Image}. These metrics are all predefined, which makes the model hard to handle some radical cases, i.e., imbalanced or noisy datasets~\cite{su2019strong,Lin2017Clustering}. Similarity can also be constructed by using data-driven techniques\cite{Lin2014A}. The sparse subspace clustering extends similarity-based methods to enable them to handle more general situations with the data-driven technique~\cite{elhamifar2013sparse}. It utilizes the self-expressive property of data to map samples into locally linear space, within which similarity is constructed to fulfill the clustering task. The sparse subspace clustering achieves excellent clustering performance when the samples lie in a union of low dimensional subspaces but fails when this assumption is violated.

 Alternatively,  some researches attempt to utilize affinities among multiple samples rather than pairs to learn a high order similarity. For example, the hypergraph clustering proposes to characterize a high order relationship among samples, thereby enhancing the clustering performance by preserving a local manifold among samples~\cite{ghoshdastidar2017uniform,ochs2012higher}. It generated hyperedge with two or more samples and introduced a new form of high order similarity~\cite{huang2011unsupervised}. Nevertheless, it has been shown that a hypergraph could degenerate into a standard graph. Incited by this observation, recent works have developed various extensions on the hypergraph clustering~\cite{liu2010robust,jain2013efficient,li2014context}.

Different from the pairwise similarity, the hypergraph clustering models depend on the high order similarity by averaging all pairwise messages within hyperedges. Therefore, the learned high order relationship relies heavily on pairwise information. Moreover, for samples with a large or massive dimension of features, the learned similarity shows small variations due to the averaging operations by the construction of hyperedges. Therefore, it performs unsatisfactorily in such scenarios.

\section{Integrating Tensor Similarity with Pairwise Similarity (IPS2)}\label{sec:method}

In this paper, we consider incorporating a high order similarity for enhancing the clustering performance.

We present the overview of IPS2 in Fig.~\ref{fig:framework}. Given an observed dataset, one can compute the sample-to-sample affinities and further organize them as a similarity matrix. In addition to the pairwise information, we propose to exploit high order affinities among multiple samples. In this paper, we encode the similarity for two sample pairs by a fourth-order tensor, with each entry denoting the relationship of two pairs, i.e. four samples. We thereafter extract a high order similarity different from the popular pairwise similarity, and the former can provide complementary information for the latter.
We jointly learn the pairwise and high order similarities, and perform $k$-means in the combination of them, so that the fused similarity can achieve efficient results.

\subsection{Notations and Background}\label{sub:notation}

Throughout the paper, vectors are signified by bold lowercase letters and matrices are denoted by bold uppercase letters. We let scalar variables be denoted by normal symbols for clarity. For example, given an observed dataset with $m$ samples and $n$ features $ \vX = [\vx_{1},\vx_{2},\cdots,\vx_{m}]$, let $\vS\in\mathbb{R}^{m\times m}$ denote the pairwise similarity. Each entry $\vS_{i,j}$ measures the affinity between the $i$-th and $j$-th samples.  We use $\underline{n}$ to denote a set $\{1,2,\cdots,n\}$. The similarity for two pairs are encoded in a $4$-dimensional tensor $\mathcal{T}\in\mathbb{R}^{m\times m\times m\times m}$ and each entry $\mathcal{T}_{i,j,k,l}$ denotes the relationship between pair $(\vx_i,\vx_j)$ and pair $(\vx_k,\vx_l)$. The tensor could be reorganized along a direction, named by unfolding operation.
\newtheorem{definition}{\hspace{0em}Definition}[section]
\begin{definition}\label{def:unfolding}
(\textbf{Unfolding}) Let $\mathcal{T}$ be a $4$-th order $m$-dimension tensor. It can be reorganized into a $m^2$-by-$m^2$ square matrix $\hat{\textbf{T}}$  by {\it unfolding}, with  its $(k,h)$-th entry given by
\begin{center}
\begin{equation}\label{equ:unfold}
\begin{aligned}
  \hat{\textbf{T}}_{r,s}=\mathcal {T}_{i,j,k,l}
\end{aligned}
\end{equation}
\end{center}
with $r = m(j-1)+i, s = m(l-1)+k, i,j,k,l\in \underline m.$
\end{definition}

\begin{definition}\label{definition}
(\textbf{Kronecker Product}) Let $\vA\in {\mathbb{R}}^{m\times n}$, $\vB\in {\mathbb{ R}}^{p\times q}$ be two matrices. Then the Kronecker product (or tensor
product) of $\vA$ and $\vB$ is defined as the matrix
\[  \vA\otimes \vB =
\begin{bmatrix}
a_{11}\vB &\cdots & a_{1n}\vB\\
        \vdots & \ddots &\vdots\\
        a_{m1}\vB &\cdots & a_{mn}\vB
\end{bmatrix}
\in {\it R}^{mp\times nq}
\]
\end{definition}

\subsection{Decomposable Tensor Similarity is the Kronecker Product of Pairwise Similarities}\label{decomposable}

The classical similarity-based methods are achieved by defining a  sample-to-sample similarity, whether it is predefined or data-driven. Akin to the pairwise similarity, each entry of a tensor similarity captures relationships between two sample pairs. In this spirit,  it is natural to define the entry as the product of two pairwise similarities:
\begin{equation}\label{decomposable}
\begin{aligned}
\cT_{i,j,k,l} = \vS_{i,k}\vS_{j,l}, i,j,k,l \in \underline{m}.
\end{aligned}
\end{equation}
In this case, the similarity between pair  ($\vx_i,\vx_k$) and pair  ($\vx_j,\vx_l$) is directly relied on their pairwise relationships, making it convenient to establish a link between $\cT$ and $\vS$ as the following theorem.
\newtheoremstyle{mystyle}{3pt}{3pt}{\kaishu}{0cm}{\heiti2 }{}{1em}{[section]}
\begin{theorem}\label{kron_property}
We call  $\cT$ as the decomposable tensor similarity, if its unfolding satisfies $\hat {\vT} = \vS \otimes \vS$ for any pairwise similarity $\vS$.
\end{theorem}
\begin{proof}
For the  unfolded tensor similarity $\cT$ defined in Eq.~(\ref{decomposable}),  we have
\begin{equation*}
\begin{aligned}
\hat {\vT} &= \begin{bmatrix}
        \cT_{1,1,1,1}&\cdots & \cT_{1,1,m,m}\\
        \vdots & \ddots &\vdots\\
        \cT_{m,m,1,1} &\cdots & \cT_{m,m,m,m}
        \end{bmatrix}\\
        &= \begin{bmatrix}
        \vS_{1,1}\vS_{1,1} &\cdots & \vS_{1,m}\vS_{1,m}\\
        \vdots & \ddots &\vdots\\
        \vS_{m,1}\vS_{m,1} &\cdots & \vS_{m,m}\vS_{m,m}\\
        \end{bmatrix}\\
         &= \begin{bmatrix}
        \vS_{1,1}\vS &\cdots & \vS_{1,m}\vS\\
        \vdots & \ddots &\vdots\\
        \vS_{m,1}\vS &\cdots & \vS_{m,m}\vS\\
        \end{bmatrix}
        \in {\it R}^{m^2\times m^2},\\
\end{aligned}
\end{equation*}
where element $\hat {\vT}_{i,j,k,l}$ is the product of the similarities of pairs $(\vx_i,\vx_j)$ and $(\vx_k,\vx_l)$. By the definition, it is easy to know that $\hat \vT = \vS\otimes \vS$ holds.
\end{proof}

This theorem builds a bridge to relate the tensor similarity to the well-studied pairwise similarity.

Given a similarity matrix $\vS\in \mathbb{R}^{m\times m}$, the classical spectral graph clustering method starts with construction of a Laplacian matrix as
$\vL = \vD^{-\frac{1}{2}}\vS \vD^{-\frac{1}{2}}$~(\textbf{Suppl. Section 1}), where $\vD$ is the vertex degree matrix and its diagonal element is computed by $\vD_{i,i}=\sum\limits_{j=1}^{m}\vS_{i,j}$. Then a clustering task is performed through obtaining the eigenvectors $\vv$ of the Laplacian matrix by solving $\vL \vv=\lambda \vv$.

Similarly, let $\cT$ be a decomposable tensor similarity with $\hat {\vT}$  being its unfolded matrix. We can compute its  Laplacian matrix $\hat{\vL}$ by $\hat{\vL}= \hat{\vD}^{-\frac{1}{2}}\hat{\vT} \hat{\vD}^{-\frac{1}{2}}$, where matrix $\hat{\vD}$ is the degree matrix of $\hat{\vT}$. We have the following claim,

\begin{theorem}\label{eigenvector}
Given a pairwise similarity matrix $\vS$ and a decomposable tensor similarity defined in Eq.~(\ref{decomposable}), let $\vL$ and $\hat{\vL}$ be their corresponding Laplacian matrices,  one has
\begin{eqnarray}
  \hat \vL = \vL\otimes \vL.
\end{eqnarray}

Moreover, let $\vv$ and $\hat\vv$ be the dominant eigenvector for the matrices $L$ and $\hat{\vL}$, respectively, one has
\begin{eqnarray}
  \hat \vv = \vv \otimes \vv.
 \end{eqnarray}
\end{theorem}

\begin{figure}
       \begin{minipage}[b]{.25\textwidth}
        \begin{flushright}
         \includegraphics[width=9cm]{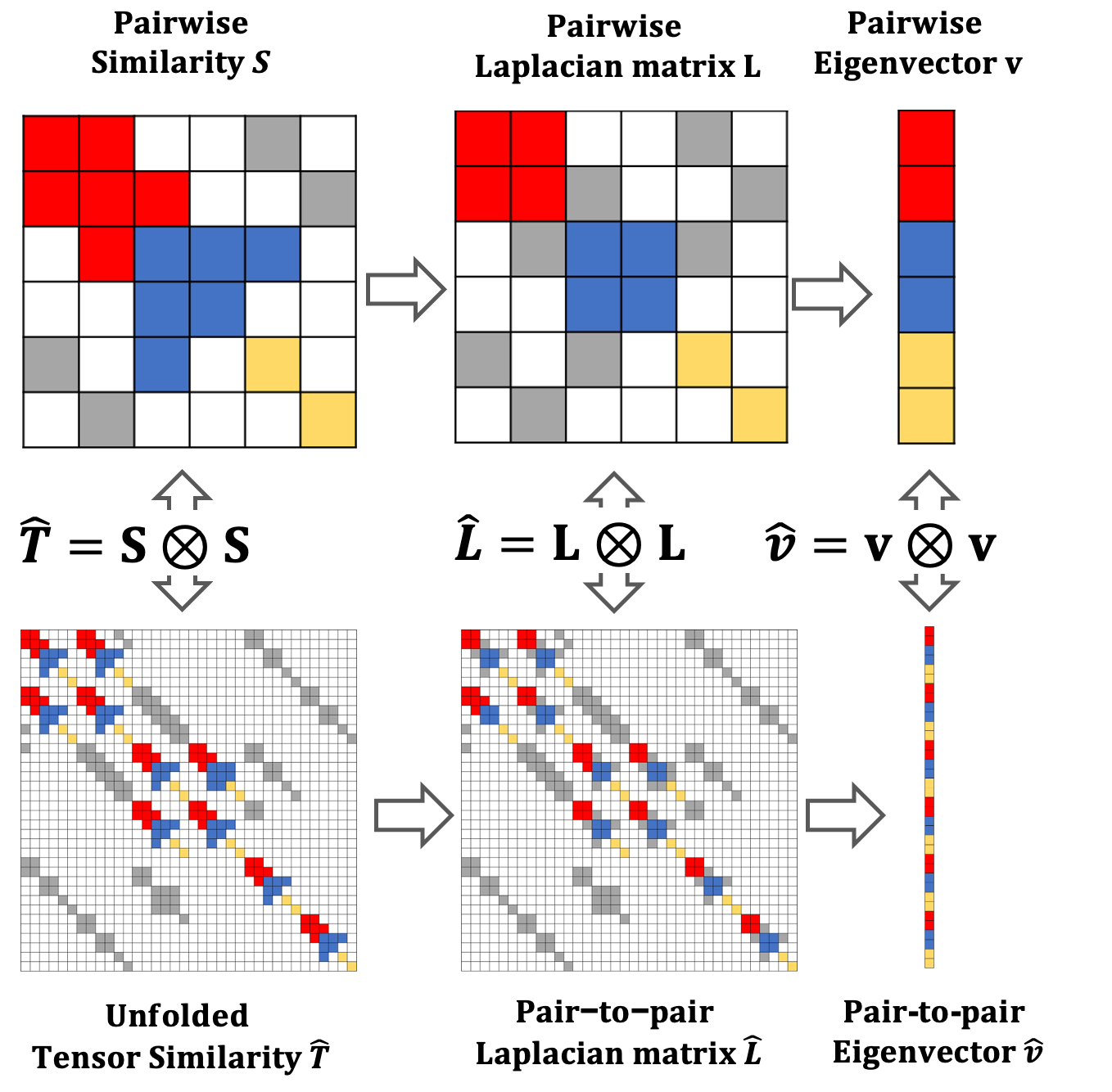}
        \end{flushright}
        \end{minipage}
   \caption{Theo.~\ref{kron_property} and Theo.~\ref{eigenvector} illustrate the connection between the pairwise similarity and the unfolded indecomposable tensor similarity.}\label{link}
\end{figure}

The detailed proof of Theo.~\ref{eigenvector} is provided in \textbf{Suppl. Section 2}. The theorem shows that the local structure extracted from pairs is determined by the Kronecker product of pairwise information of data, and pair-to-pair Laplacian matrix and its eigenvectors follow this property. In particular, as stated in literature~\cite{kumar2011co}, the pairwise eigenvector $\vv$ conveys important information about the underlying manifold of data, and clustering results can be efficiently computed based on the newly constructed similarity $\vv \vv^T$. Together with Theo.~\ref{eigenvector}, the above implies that we could construct a new high order similarity by reshaping  $\hat{\vv}$ into a $m\times m$ matrix $\vV$.

\begin{figure*}[!htb]
   \centering
   \begin{tabular}{c c c}
    \subfigure[Pairwise similarity]
        {\includegraphics[width = 2.3 in]{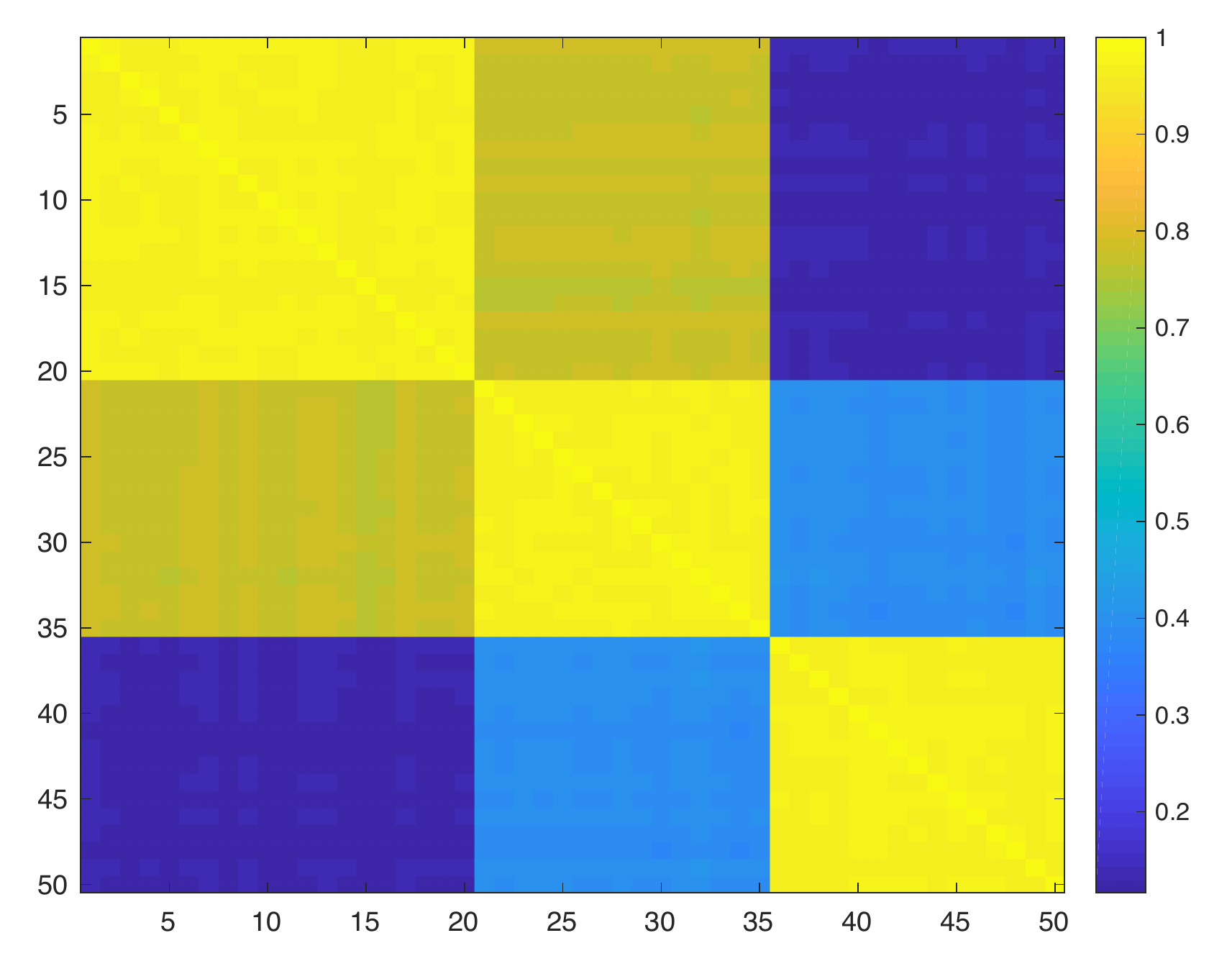}}
    \subfigure[Decomposable similarity]
        {\includegraphics[width = 2.3in]{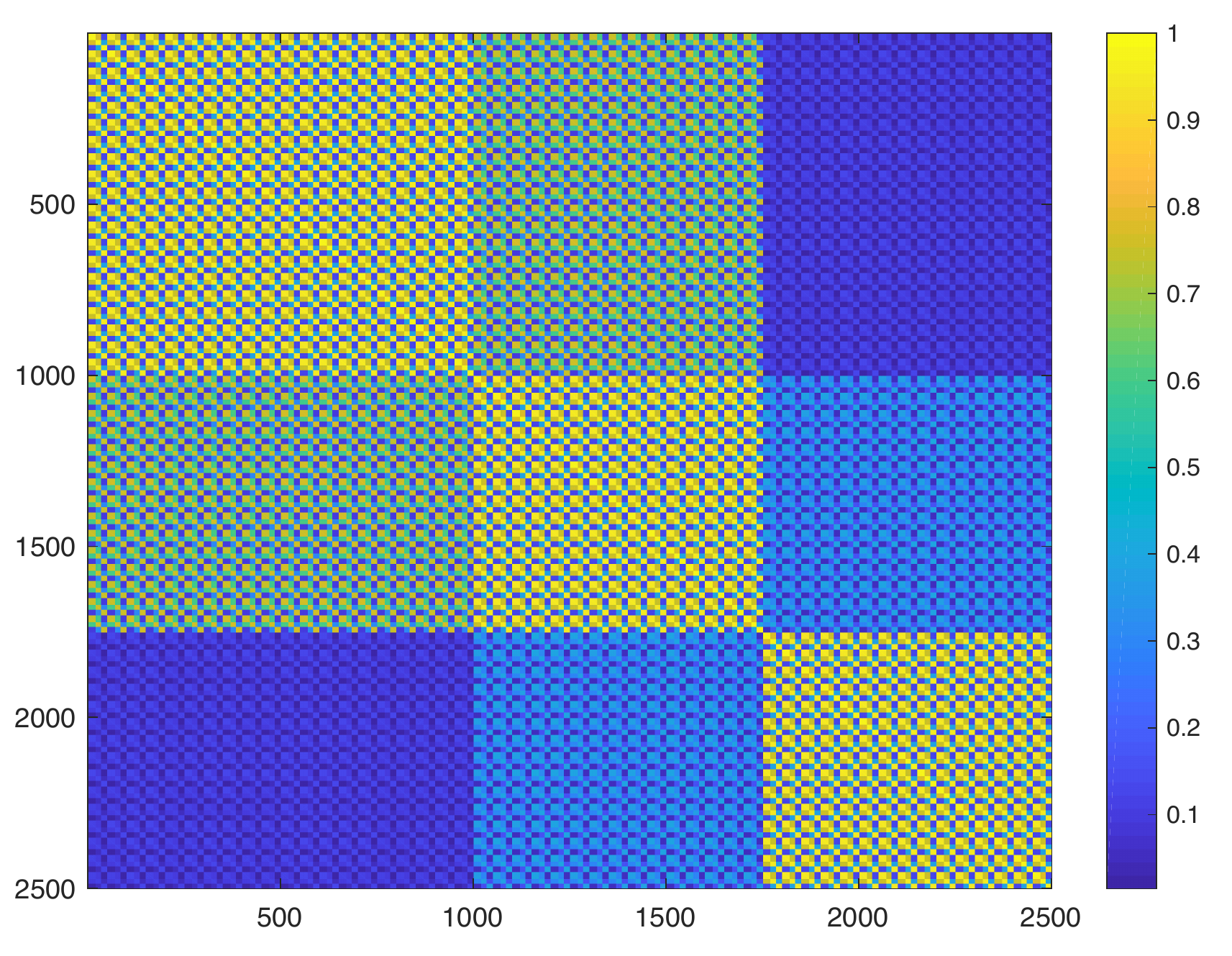}}
     \subfigure[Indecomposable similarity]
        {\includegraphics[width = 2.3 in]{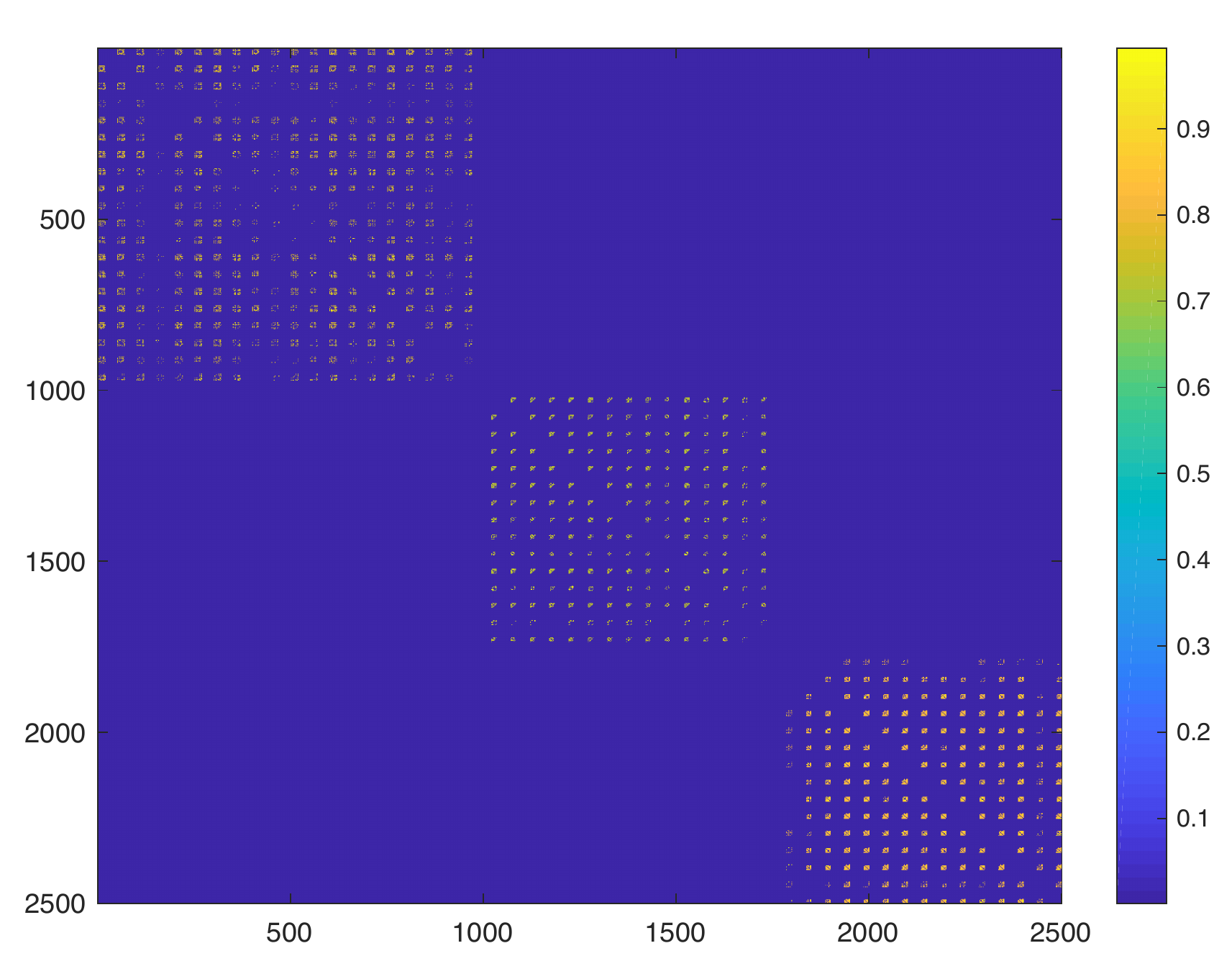}}
      \end{tabular}
   \vskip -.3cm
   \caption{The heatmaps of unfolded indecomposable tensor similarity  over a synthetic dataset. The block structures in the heatmap on   unfolded similarity are evident. This indicates that the indecomposable tensor similarity captures evident structural information than decomposable one does.}\label{fig:pairwise_decom_nondecom}\label{test_equ}
\end{figure*}

\begin{definition}\label{def:high order}
(\textbf{High Order Similarity}) The high order similarity $\vV$ is defined as:
\begin{equation}\label{equ:fold}
{\vV}_{i,j}=\hat{\textbf{v}}_l
\end{equation}
where $j =l - m(i-1)$ and $i \in \underline{m}$~for $l \in \underline{m^2}$.
\end{definition}

An illustration of the aforementioned results is drawn in Fig.~\ref{link}. We now note that
the difference between the pairwise similarity $\vS$ and the obtained high order similarity $\vV$ is inconsequential if the later is from a decomposable tensor similarity. This result also confirms that the obtained high order similarity $\vV$ in Eq.~(\ref{equ:fold}) suffers from the same drawbacks as the popular pairwise similarity does. They are vulnerable to noise corruptions and indifferentiable for samples with large feature dimensions~\cite{sarkar2019perfect}. We now turn to focus on the indecomposable tensor similarity and show that the high order similarity extracted from it provides complementary information for which the pairwise one may miss.

\begin{figure*}[!htb]
   \centering
    \subfigure[Pairwise similarity]
     {\includegraphics[width=2.3in]{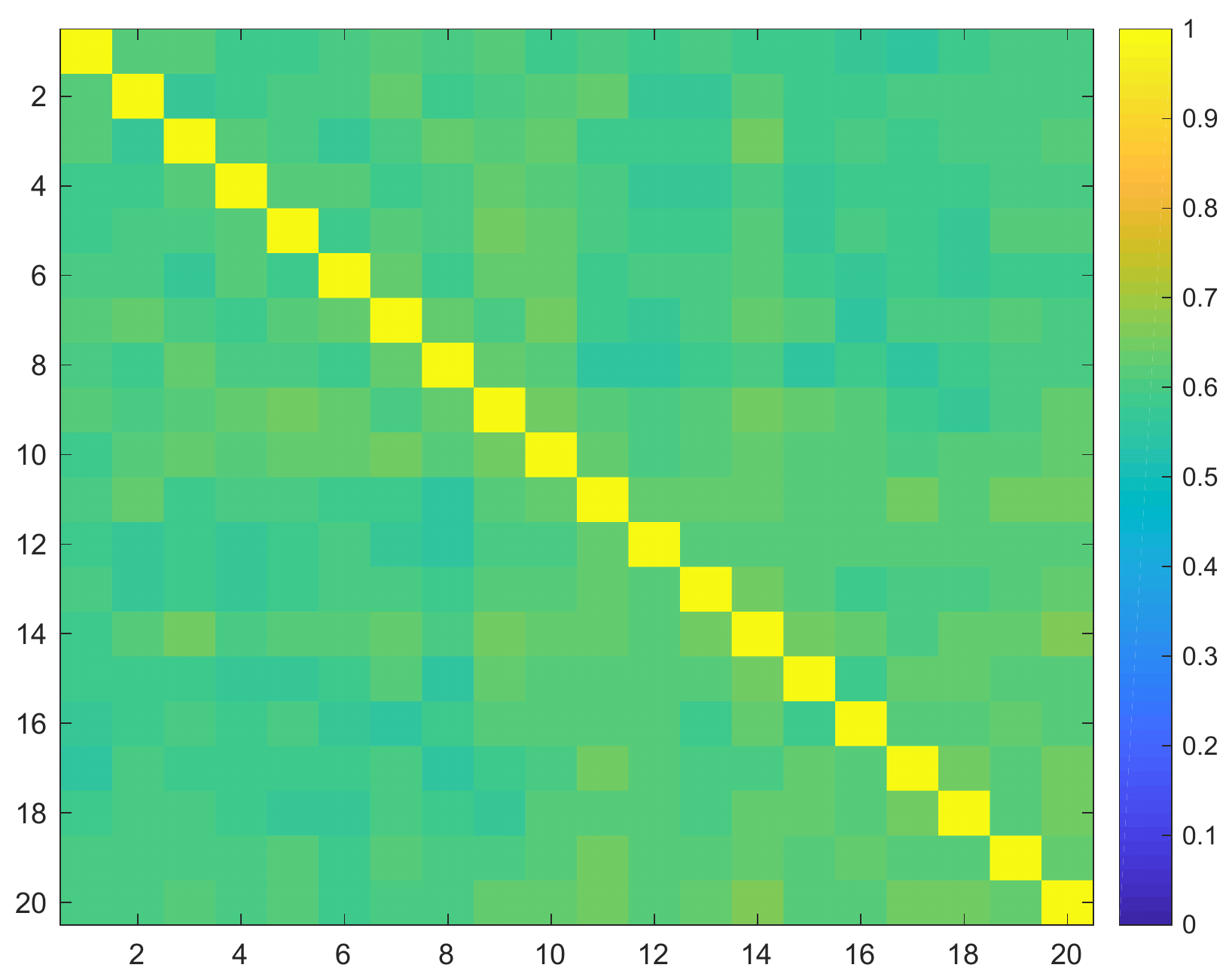}}
    \subfigure[High order similarity]
      {\includegraphics[width=2.3in]{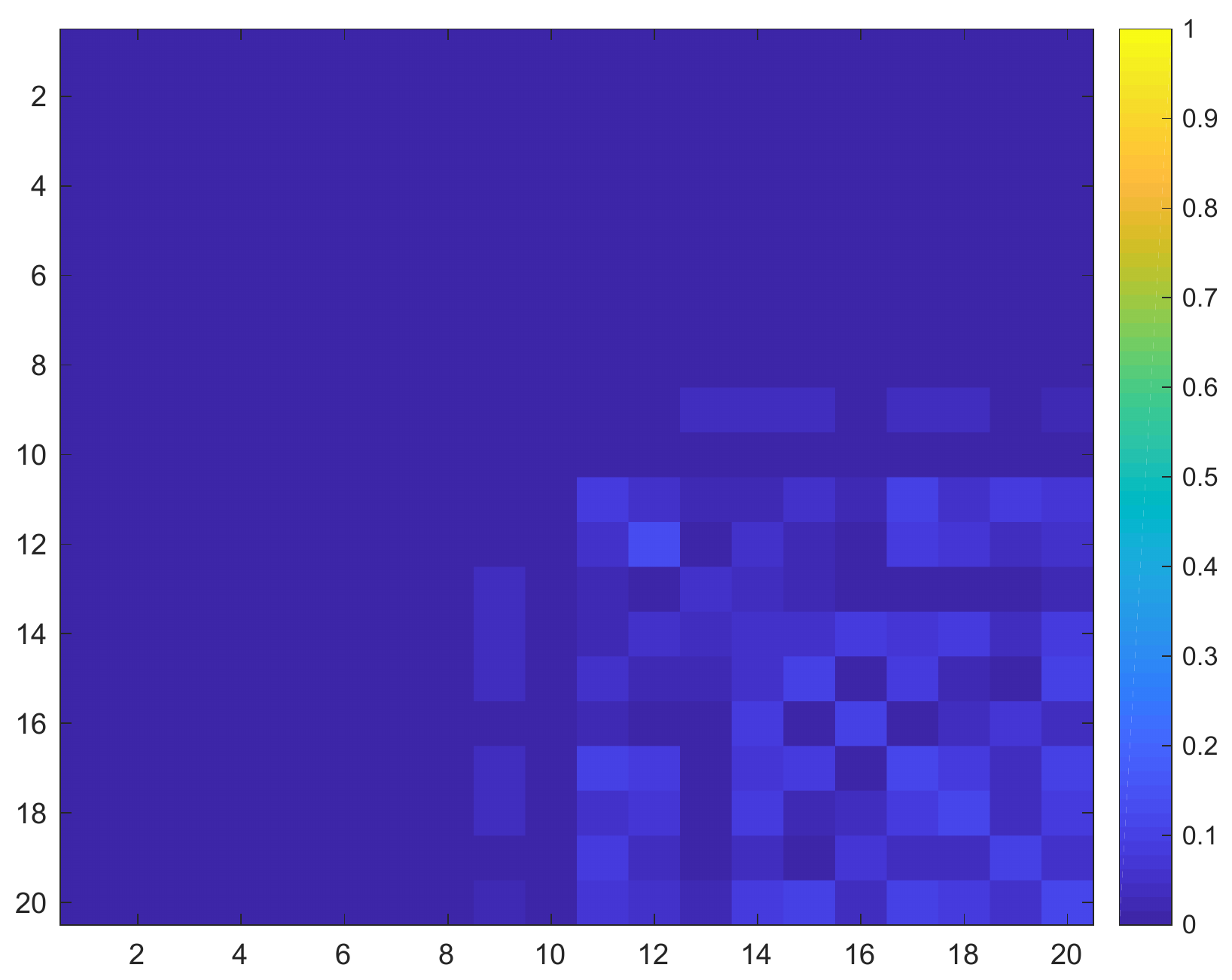}}
         \subfigure[Fused similarity]
     {\includegraphics[width=2.3in]{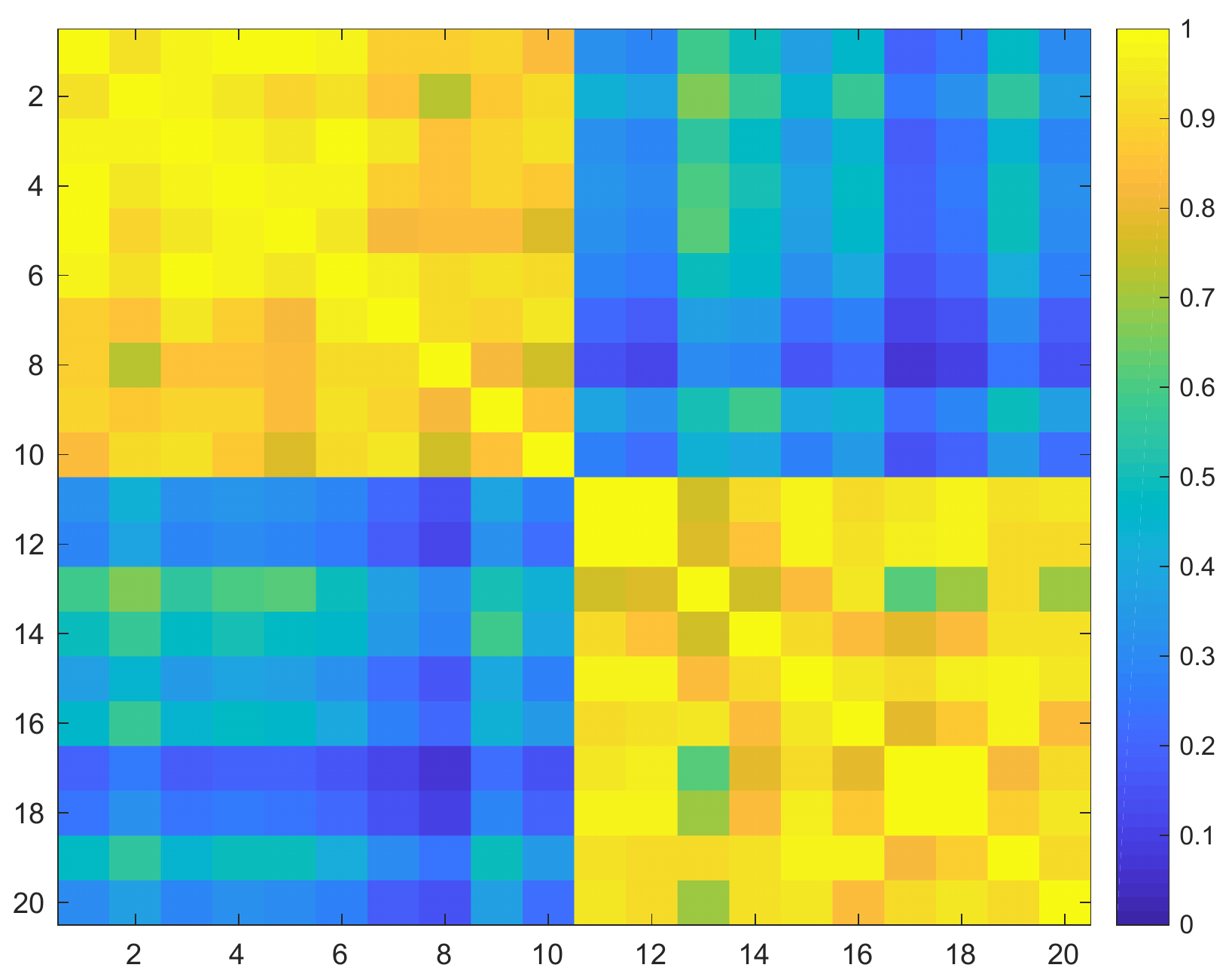}}
    \subfigure[Pairwise frequency distribution]
     {\includegraphics[width=2.3in]{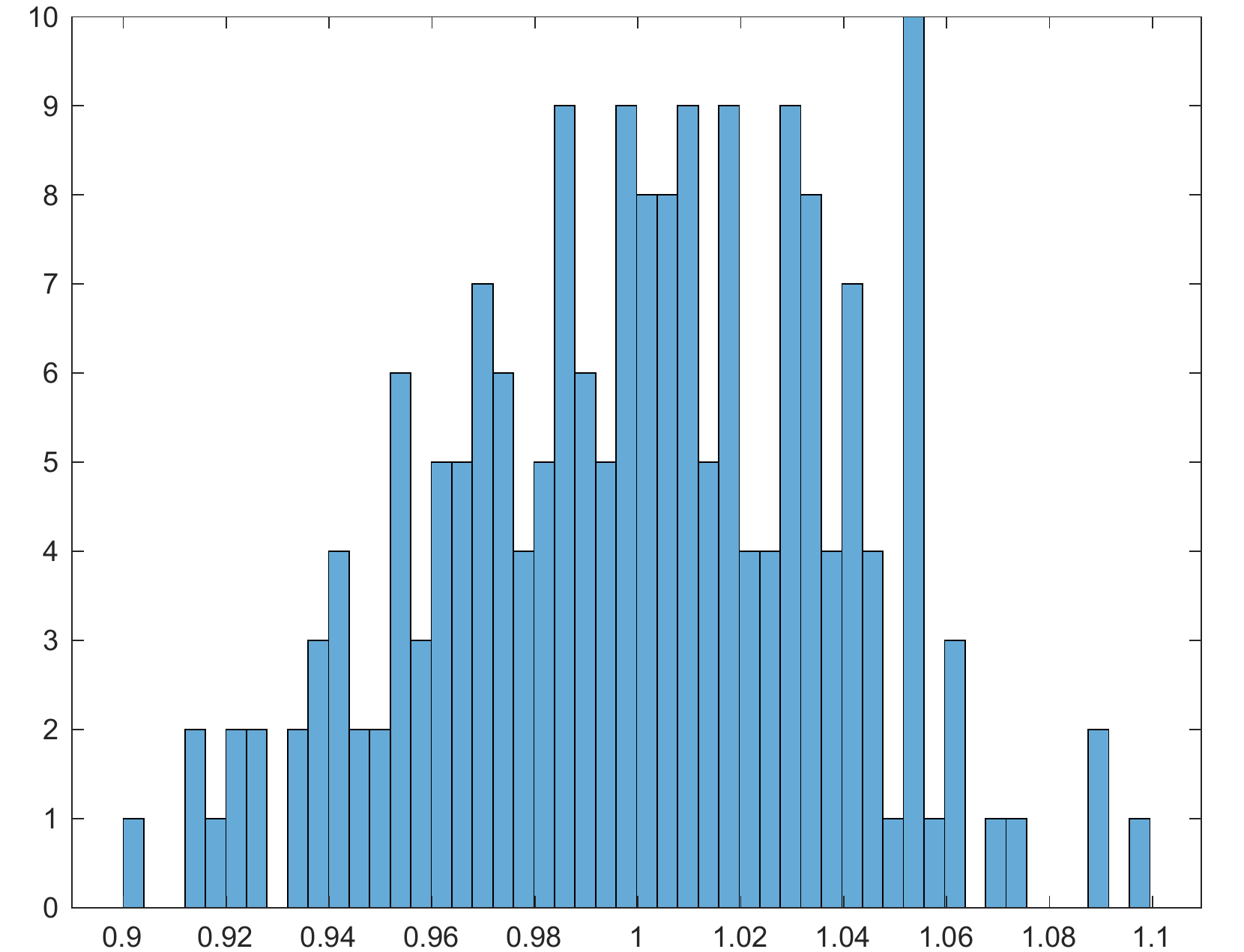}}
    \subfigure[High order frequency distribution]
      {\includegraphics[width=2.3in]{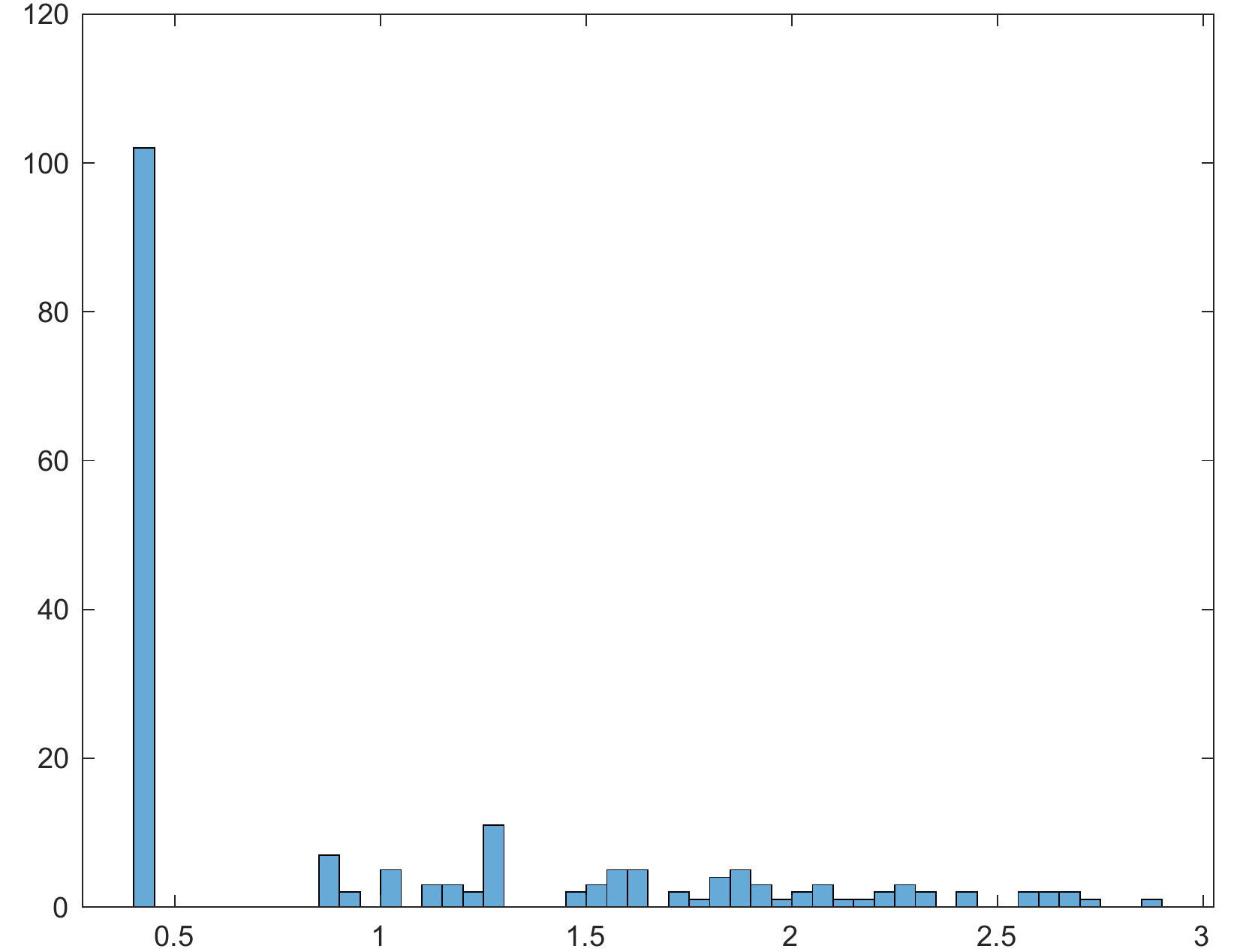}}
         \subfigure[Fused frequency distribution]
     {\includegraphics[width=2.3in]{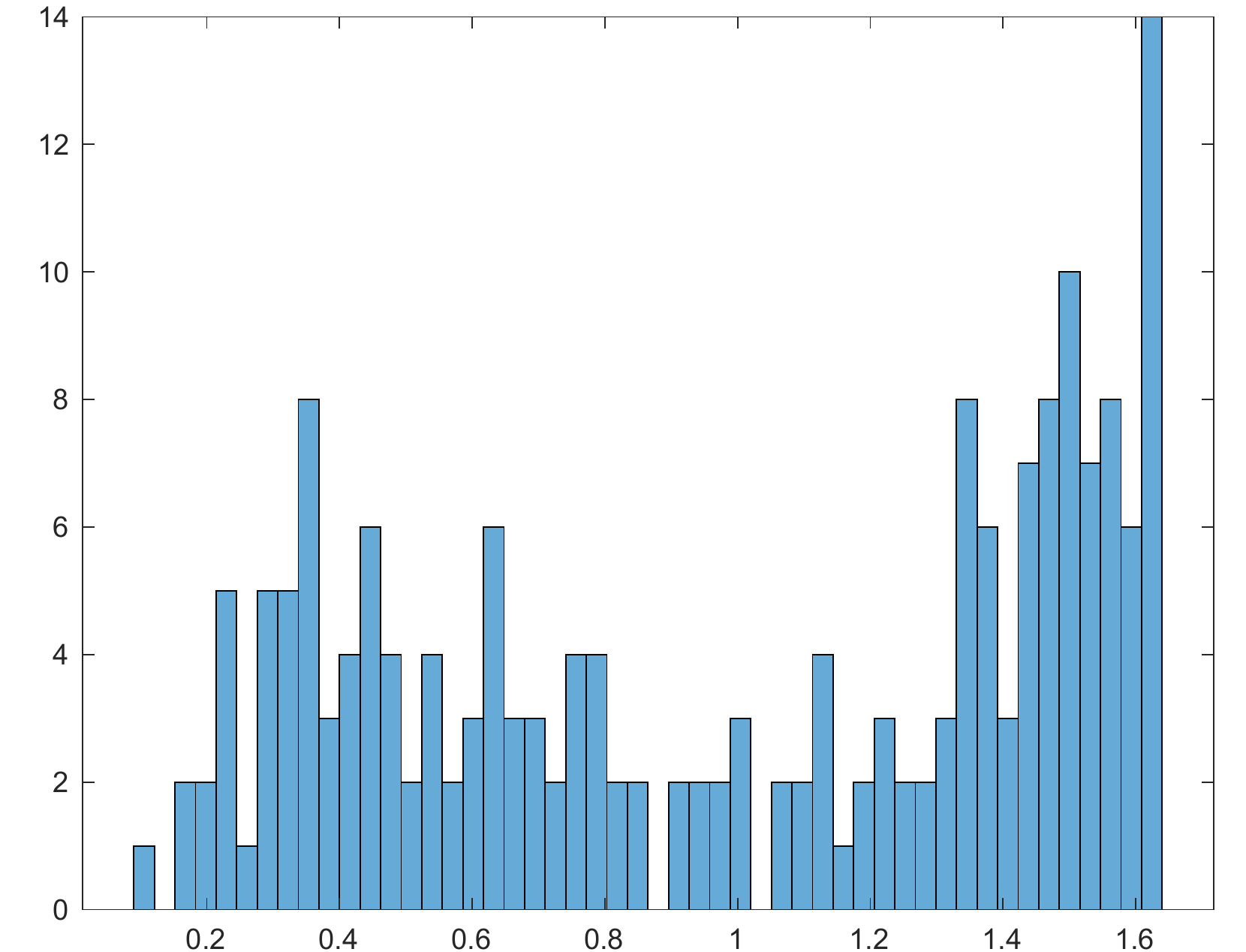}}
   \caption{Synthetic example for validating the complementarity of the high order similarity. This first row shows the heatmaps of the pairwise similarity, the high order similarities and their summation. The second row shows  the corresponding frequency distributions of the three similarities, respectively. This salient block structure and two peaks in the frequency distribution of the summation suggests that the high order similarity indeed captured information which is beneficial to the pairwise similarity.}\label{fig:frequency}
\end{figure*}

\subsection{Indecomposable Tensor Similarity Provides Complementary Information for Pairwise Similarity}\label{measurement}

The primary purpose of the tensor similarity is to supplement information which a pairwise similarity fails to provide, thereby improving the accuracy and robustness in clustering task. To achieve the goal, we propose to introduce an indecomposable tensor similarity for pairs. The idea is to capture the stable complementary information of data by encoding links among pairs. \emph{Different from the decomposable case, the indecomposable similarity plays a dual role: it not only characterizes the local manifold among the pairs but also provides the complementary information which the pairwise similarity fails to obtain.}
In view of this spirit, the indecomposable tensor similarity used in this paper is defined by:
\begin{equation}\label{equ:nondeomposable}
\begin{aligned}
{\cT}_{i,j,k,l} = \exp(-\sigma \frac{d_{i,j}+d_{k,l}}{d_{i,k}+d_{j,l}+\varepsilon})
\end{aligned}
\end{equation}
for $i,j,k,l \in \underline{m}$, where $d_{i,j}$ denotes the distance between samples $\vx_i$ and $\vx_j$, i.e. Euclidean distance. Parameter $\varepsilon$ is a given small parameter less than 0.001 to overcome the instability caused by a zero denominator. The parameter $\sigma$ is an empirical parameter. Entry ${\cT}_{i,j,k,l}$ refers to the similarity between pair $(\vx_{i},\vx_{k})$ and pair $(\vx_{j},\vx_{l})$. It models both relationships within pairs as well as the geometric links between pairs. It achieves its minimum value when the pair $(\vx_i,\vx_k)$ falls into a cluster $c$, while the pair $(\vx_j,\vx_l)$ falls into a distinct cluster $c'$.

Once we have the indecomposable tensor similarity $\cT$, we can unfold it into a square matrix $\hat\vT$. Similar to the decomposable case, we compute its dominant eigenvector of the Laplacian matrix and reshape it into a $m\times m$ dimension matrix $\vV$. The obtained matrix is now viewed as a high order similarity extracted from the indecomposable tensor.

One of the main challenges for pairwise similarity is its vulnerability to noisy datasets. It is important to stress that the indecomposable tensor similarity can confront this challenge both in two sides. To demonstrate such characteristics, we create a synthetic dataset for signifying its merits.

The synthetic dataset composes of three clusters, drawn from i.i.d Normal distributions with an equal standard deviation of 0.5 and the different mean of 1, 2.5 and 5, respectively. The dataset consists of 50 samples, each with a feature dimension of 500. We note that there are small differences in the computed pairwise similarity. The histograms on the pairwise similarity, unfolded decomposable and indecomposable tensor similarity are shown in Fig.~\ref{fig:pairwise_decom_nondecom}. The off-diagonal areas in Fig.~\ref{fig:pairwise_decom_nondecom} (a) and (b) are shown to be similar to the diagonal areas. This result implies that clustering based on the pairwise similarity or decomposable tensor similarity was heavily affected. In comparison, the indecomposable tensor similarity is shown to contain distinct components as Fig.~\ref{fig:pairwise_decom_nondecom} (c) shown. The off-diagonal and diagonal areas are separated. The indecomposable tensor similarity depicts geometric links between the pairs and thus provides complementary information which the pairwise similarity missed.

To verify such complementarity and thereby have a deep understanding of the similarity extracted from the indecomposable tensor, we modified the previous synthetic data by keeping only two clusters. Heatmaps of the pairwise, extracted high order similarities and their summation, are shown in the first row of Fig.~\ref{fig:frequency}. The frequency distributions on the similarity values of the three types are calculated and shown in the second row of Fig.~\ref{fig:frequency}. Due to the small sample size with large feature dimension, notoriously known as {\it curse of dimension}, the histogram distribution of the pairwise similarity behaves like Normal distribution, resulting in that two clusters in the pairwise similarity have no apparent boundary (Fig.~\ref{fig:frequency}(a)). Thus, it is no wonder that there is no distinction between the two clusters in the pairwise frequency distribution~(Fig.~\ref{fig:frequency}(d)).
In comparison, the high order similarity has an apparent block structure. One could find a clear boundary in its histogram in Fig.~\ref{fig:frequency}(e), implying another cluster. These differences demonstrate that the indecomposable tensor similarity could provide information that is ignored by the pairwise similarity. Furthermore, when the pairwise similarity and the extracted high order similarity are fused by calculating their average, the fused result has two apparent blocks~(Fig.~\ref{fig:frequency}(c)), which leads to two separated peaks in its frequency distribution~(Fig.~\ref{fig:frequency}(f)). The surprising results also confirmed superior clustering performance by exploiting both the pairwise and extracted high order similarity simultaneously.

\subsection{Similarity Fusion by Averaging the Pairwise and the Extracted High Order Similarities}\label{sec:information fusion}

In the previous section, we demonstrated that the indecomposable tensor similarity could provide complementary information for a pairwise similarity. Combining with the extracted high order information learning from pairs, we aim to enhance the clustering performance of methods based on a pairwise similarity to deal with scenarios in which it yields poor performance.

To enjoy the merits of both extracted high order and pairwise similarities, we propose to average them for clustering directly. Let matrices $\vS$ and $\vV$ be pairwise and the high order similarities respectively, and the fused similarity is calculated by $\textbf{U}= \frac{1}{2}(\vS +\vV)$ which is used to represent the affinities of samples for the usage of clustering.

Finally, we would like to emphasize that there are numerous ways to define or learn the indecomposable tensor similarity, as the popular pairwise similarity does. Also, there are different approaches to fuse them. Here, we only used a simple approach to demonstrate the effectiveness of the current strategy.

\begin{algorithm}[htb]
\caption{ Algorithm for Solving IPS2}\label{tab:algorithm}
\begin{algorithmic}[1]
\Require
Dataset with $m$ samples $\vX=[\vx_1,\vx_2,...,\vx_m]$;
Number of clusters $c$;
Number of nearest neighbours $k$;
\Ensure
The clustering labels;
\State Constructing a pairwise similarity matrix of data $\vS$ by Gaussian kernel function;
\State Constructing a sparse tensor similarity $\cT$ and unfolding it into matrix $\hat {\vT}$. Then compute the corresponding Laplacian matrix $\hat {\vL}$;
\State Reshaping the top $c$ eigenvectors of $\hat {\vL}$ into matrices $\vV_i, i\in \underline{c}$;
\State Calculating the integrated similarity matrix $\vV$ by averaging all {$\vV_i$};
\State Computing the summation of $\vS$ and $\vV$ and running k-means on it to obtain clustering labels;
\end{algorithmic}
\end{algorithm}

\subsection{Numerical Scheme for IPS2}\label{optimization}

In this section, we provide a numerical scheme for solving the proposed IPS2 algorithm.

Suppose that we have a raw dataset, the first step is to construct a non-decomposable tensor similarity $\cT$ by calculating relationships among pairs. It has a size of $m^4$ for $m$ samples. To reduce the computational cost, the top $k$ nearest neighbors of each sample are stored and used to calculate the tensor similarity. The remaining entries, their values are assigned to be zero. Thus, the obtained tensor similarity $\cT$ is sparse. It is then unfolded to have a large square matrix of $\hat{\vT}$. From which, the Laplacian matrix $\hat{\vL}$ is constructed, and the top $c$ dominant eigenvectors are calculated. The eigenvectors are reshaped into square matrices, and thereby we obtained the high order similarities $\vV_1,\vV_2,\cdots, \vV_c$. They are fused with the pairwise similarity by simple averaging to yield a fused similarity. Finally, the clustering task is achieved by using benchmark methods, such as $k$-means to obtain the sample assignment.

We sketch the pseudocode of the algorithm in Alg.~\ref{tab:algorithm}. To verify that the high order similarity serves as complementary information for the pairwise one, we develop the algorithm which only uses high order information for clustering and denote it as PPC. The pseudocode of PPC is the same as IPS2 but without step 5.

\section{EXPERIMENTAL RESULTS}\label{Results}

We evaluated the performance of the proposed method on four real-world datasets, 'Soybean', 'SCADI', 'BBCSports', and 'Yale'. We downloaded Soybean and SCADI datasets from \emph{UCI Machine Learning Repository~(https://archive.ics.uci.edu/ml/index.php)}, BBCSports dataset from \emph{Insight Resources~{http://mlg.ucd.ie/datasets/bbc.html}}, and Yale dataset from \emph{http://vision.ucsd.edu/content/yale-face-database}. To elaborate the ability of our method tackling class-imbalance and noise contamination, we further constructed two challenging synthetic datasets, and the description of them are given in the following. Finally, the analysis of time complexity and parameter pruning is provided.

\subsection{Evaluation Metrics and Baseline Methods}

We used five widely-used metrics to evaluate the performance of the proposed method~(\textbf{Suppl. Section 3}): accuracy (ACC), adjusted rand index (ARI), F-score(F-SCORE), normalized mutual information (NMI), and purity (PURITY). The larger the value gets, the better the performance is.

For a fair comparison, we employed five representative methods to serve as a baseline. A brief introduction and the parameter settings of these approaches are as follows:

\begin{itemize}
 \item Spectral clustering (\textbf{SC}) \cite{ng2002spectral}: the classic spectral clustering gives a baseline on behalf of pairwise similarity.

 \item Sparse subspace clustering (\textbf{SSC}) \cite{elhamifar2013sparse}: SSC introduces the self-expressiveness property of data to study the sparse representation of each point in the dictionary consisting of other points. We set parameters according to the suggestion of the authors.

 \item Clique averaging+ncut (\textbf{CAVE})\cite{agarwal2005beyond}: The hypergraph, on behalf of the high order similarity, is approximated using the clique averaging, and the resulting graph is partitioned using the normalized cuts algorithm.

 \item Finding hypergraph communities (\textbf{HGC}) \cite{vazquez2009finding}: A hypergraph generative model with a built-in group structure and with the variational Bayes algorithm for finding network communities. We set parameters according to the authors' suggestion.

 \item Pair-to-pair clustering (\textbf{PPC}): Applying the spectral clustering method on our proposed high order similarity only. It serves as a baseline to verify the complementarity of the high order similarity to the pairwise similarity.

 \item Integrating tensor similarity and pairwise similarity (\textbf{IPS2}): The proposed method.
\end{itemize}

\subsection{Experiments on Real-world Datasets to Demonstrate the Superiority of IPS2}

We verified the effectiveness of IPS2 on four real-world datasets. The data statistics are summarized in Tab.~\ref{tab2:data}. A detailed description of them are as follows:

\begin{itemize}
\item\textbf{Soybean} dataset is a small subset of the soybean database containing 47 distinct instances. Each subject has 35 attributes to describe its growth indicators.

 \item\textbf{SCADI} dataset contains 206 attributes of 70 children with physical and motor disability based on ICF-CY. The first 205 attributes record self-care activities based on ICF-CY. The last column 'Class' is the label, which refers to the presence of the self-care problems of the children with physical and motor disabilities.

 \item\textbf{BBCSports} dataset consists of 282 documents from the BBC Sports website. All documents correspond to sports news articles in five topical areas from 2004 to 2005 and extract 2,544 features of it as our dataset.

 \item\textbf{Yale} dataset contains 165 grayscale images in GIF format of 15 individuals. Each individual has 11 facial expressions or configuration images: center-light, with glasses, happy, left-light, without glasses, normal, right-light, sad, sleepy, surprised, and the wink. We chose 5 classes of all images and extracted 4,096-dimensional raw pixel values of 55 images for our experiments.
\end{itemize}

Except for the Soybean data, the number of the feature dimensionality is larger than that of the sample size. Such imbalance within these datasets brings challenges to achieve an accurate clustering. For a fair comparison, we ran each method 50 times and calculated their mean values and standard deviations on the five measurements stated in Tab.~\ref{results}. In term of the type of the similarity, six methods can be divided into three categories, i.e. 'pairwise', 'high order' and 'fused'. The best results are highlighted in bold.

\begin{table}[!t]
  \centering
  \caption{Statistics on  four real-world datasets}\label{tab2:data}
  \renewcommand{\arraystretch}{1.5}
    \begin{tabular}{cccc}
    \hline
    Dataset & Samples  & Features & Clusters \\
       \hline
           Soybean         &47 & 35 & 4  \\
           SCADI           &70 & 200 & 7\\
           BBCSports & 60 & 2544 & 3 \\
           Yale    & 55 & 4096 & 5 \\
    \hline
    \end{tabular}
\end{table}

\begin{table*}[!t]
  \centering
  \begin{threeparttable}
  \caption{Clustering performance \protect (mean $\pm$ standard deviation) of all six methods}\label{results}
  \renewcommand{\arraystretch}{1.5}
    \begin{tabular}{c|c|c|c|c|c|c}
    \toprule
    \multirow{2}*{Dataset} &SC & SSC & CAVE & HGC & PPC & IPS2\\
    \cline{2-7}
    \multirow{2}*{} &\multicolumn{2}{c|} {Pairwise}& \multicolumn{3}{c|} {High order} & Fused\\
   \bottomrule
    \multicolumn{7}{c} {ACC}\\
    \midrule
    Soybean &0.787$\pm$0.031 &0.542$\pm$0.063 &0.894$\pm$0.030 &0.617$\pm$0.030 &0.809$\pm$0.071 &\textbf{0.936$\pm$0.040}\\
    SCADI  &0.860$\pm$0.031 &0.296$\pm$0.055  &0.564$\pm$0.007 &0.371$\pm$0.030  &0.842$\pm$0.077  &\textbf{0.877$\pm$0.040}\\
    BBCSports &0.533$\pm$0.086   &0.492$\pm$0.073 &0.617$\pm$0.030  &0.483$\pm$0.030 &0.617$\pm$0.093 &\textbf{0.800$\pm$0.100}\\
    Yale &0.491$\pm$0.081 &0.428$\pm$0.072 &0.455$\pm$0.030 &0.400$\pm$0.030 &0.473$\pm$0.079 &\textbf{0.600$\pm$0.069} \\
       \bottomrule
    \multicolumn{7}{c} {ARI}\\
    \midrule
    Soybean &0.576$\pm$0.090 &0.205$\pm$0.073 &0.739$\pm$0.090 &0.361$\pm$0.090  &0.610$\pm$0.193 &\textbf{0.829$\pm$0.119}\\
    SCADI  &0.701$\pm$0.090 &0.088$\pm$0.053  &0.430$\pm$0.005 &0.107$\pm$0.090  &0.644$\pm$0.193 &\textbf{0.761$\pm$0.119}\\
    BBCSports &0.055$\pm$0.100   &0.080$\pm$0.083 &0.152$\pm$0.090   &0.035$\pm$0.090 &0.365$\pm$0.115 &\textbf{0.523$\pm$0.132}\\
    Yale &0.214$\pm$0.080 &0.124$\pm$0.088 &0.190$\pm$0.090 &0.103$\pm$0.090 &0.150$\pm$0.089 &\textbf{0.247$\pm$0.061} \\
       \bottomrule
    \multicolumn{7}{c} {F-SCORE}\\
    \midrule
    Soybean &0.680$\pm$0.057 &0.419$\pm$0.058 &0.808$\pm$0.584 &0.539$\pm$0.058  &0.711$\pm$0.124 &\textbf{0.874$\pm$0.076}\\
    SCADI  &0.812$\pm$0.057 &0.300$\pm$0.049   &0.280$\pm$0.058 &0.775$\pm$0.124  &0.842$\pm$0.077  &\textbf{0.850$\pm$0.026}\\
    BBCSports &0.484$\pm$0.061   &0.419$\pm$0.058 &0.480$\pm$0.584  &0.369$\pm$0.058&0.581$\pm$0.069 &\textbf{0.680$\pm$0.081}\\
    Yale &0.367$\pm$0.066 &0.316$\pm$0.067 &0.341$\pm$0.079 &0.284$\pm$0.058 &0.307$\pm$0.074 &\textbf{0.387$\pm$0.051} \\
       \bottomrule
    \multicolumn{7}{c} {NMI}\\
    \midrule
    Soybean &0.728$\pm$0.067 &0.314$\pm$0.075  &0.813$\pm$0.079 &0.316$\pm$0.079 &0.803$\pm$0 .156 &\textbf{0.883$\pm$0.086}\\
    SCADI  &0.694$\pm$0.067 &0.261$\pm$0.049  &0.340$\pm$0.079 &0.657$\pm$0.156  &0.842$\pm$0.077  &\textbf{0.741$\pm$0.086}\\
    BBCSports &0.240$\pm$0.111   &0.135$\pm$0.088 &0.391$\pm$0.079  &0.109$\pm$0.079 &0.374$\pm$0.113 &\textbf{0.539$\pm$0.129}\\
    Yale &0.348$\pm$0.067 &0.271$\pm$0.096 &0.359$\pm$0.030 &0.223$\pm$0.079 &0.298$\pm$0.080 &\textbf{0.376$\pm$0.056} \\
       \bottomrule
    \multicolumn{7}{c} {PURITY}\\
    \midrule
    Soybean &0.787$\pm$0.030 &0.573$\pm$0.055 &0.894$\pm$0.030 &0.617$\pm$0.030 &0.809$\pm$0.077 &\textbf{0.936$\pm$0.040}\\
    SCADI  &0.860$\pm$0.030   &0.547$\pm$0.041  &0.777$\pm$0.007 &0.571$\pm$0.030  &0.842$\pm$0.077  &\textbf{0.877$\pm$0.040}\\
    BBCSport &0.533$\pm$0.072   &0.497$\pm$0.072 &0.617$\pm$0.030  &0.483$\pm$0.030 &0.683$\pm$0.090 &\textbf{0.800$\pm$0.098}\\
    Yale  &0.527$\pm$0.078 &0.445$\pm$0.070 &0.509$\pm$0.584 &0.418$\pm$0.030 &0.473$\pm$0.075 &\textbf{0.600$\pm$0.066} \\
    \bottomrule
    \end{tabular}
     \end{threeparttable}
\end{table*}

\begin{figure*}[!htb]
\centering
  \subfigure[SC-BBCSports]
    {\includegraphics[width=2.3in]{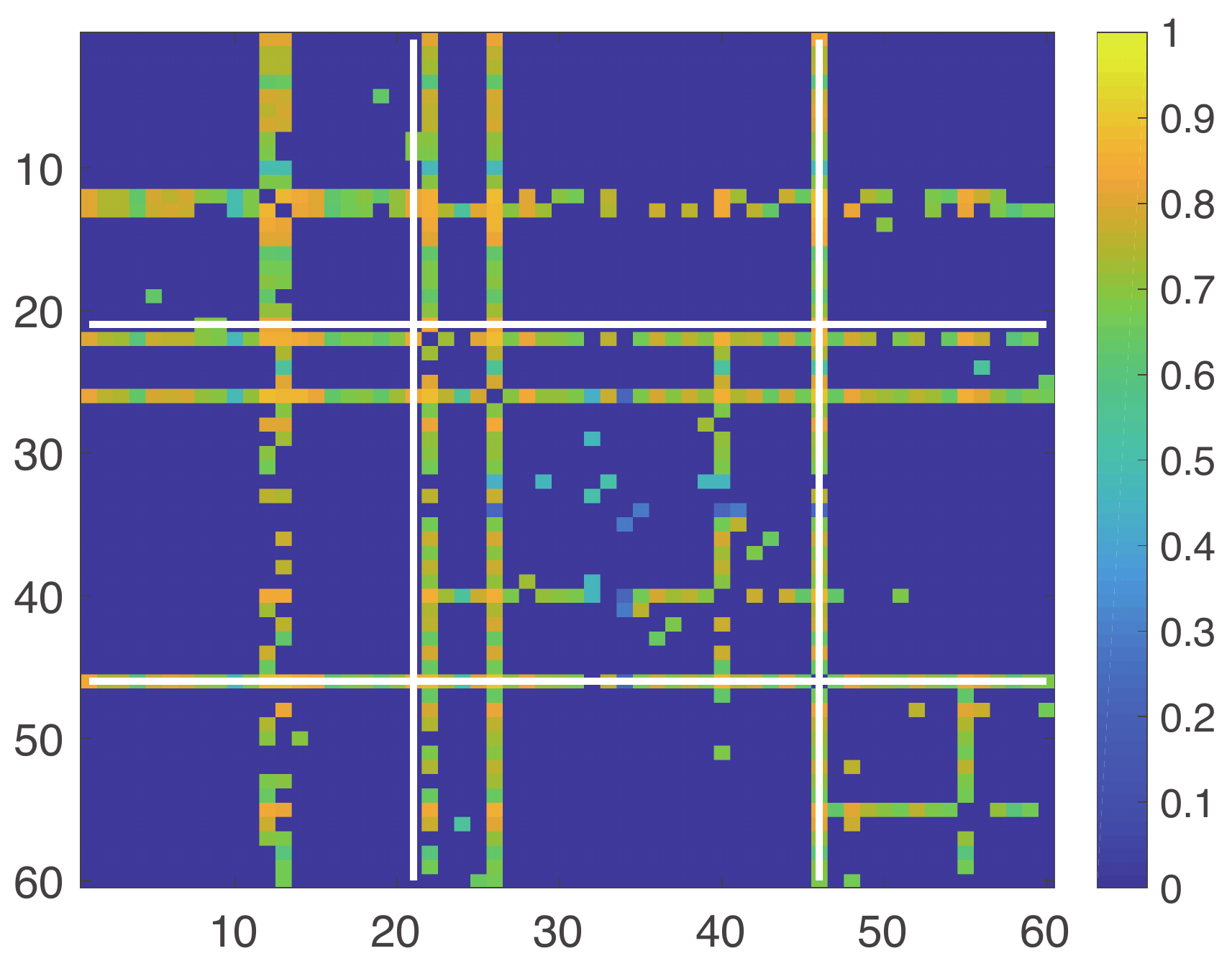}}
  \subfigure[PPC-BBCSports]
    {\includegraphics[width=2.3in]{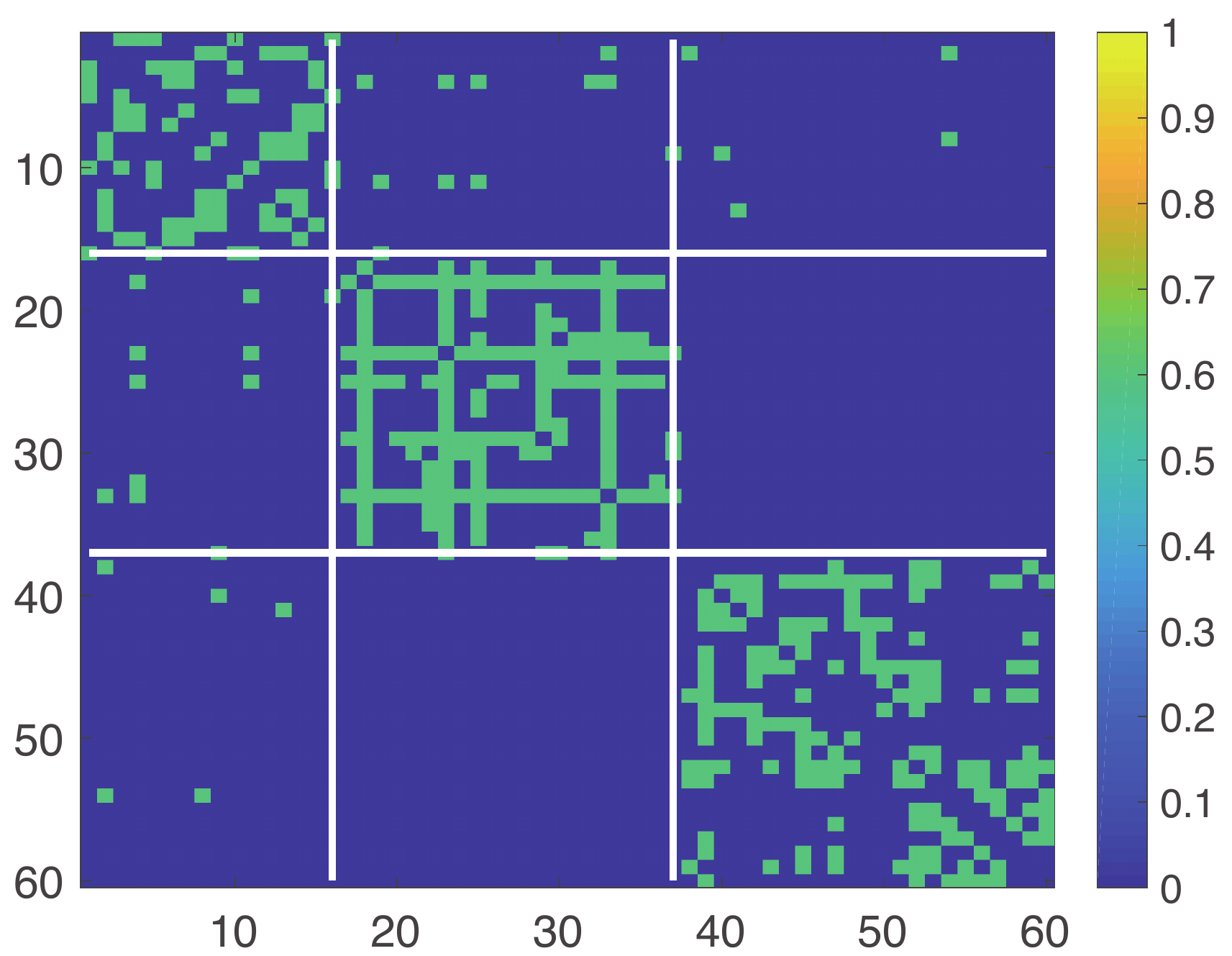}}
  \subfigure[IPS2-BBCSports]
    {\includegraphics[width=2.3in]{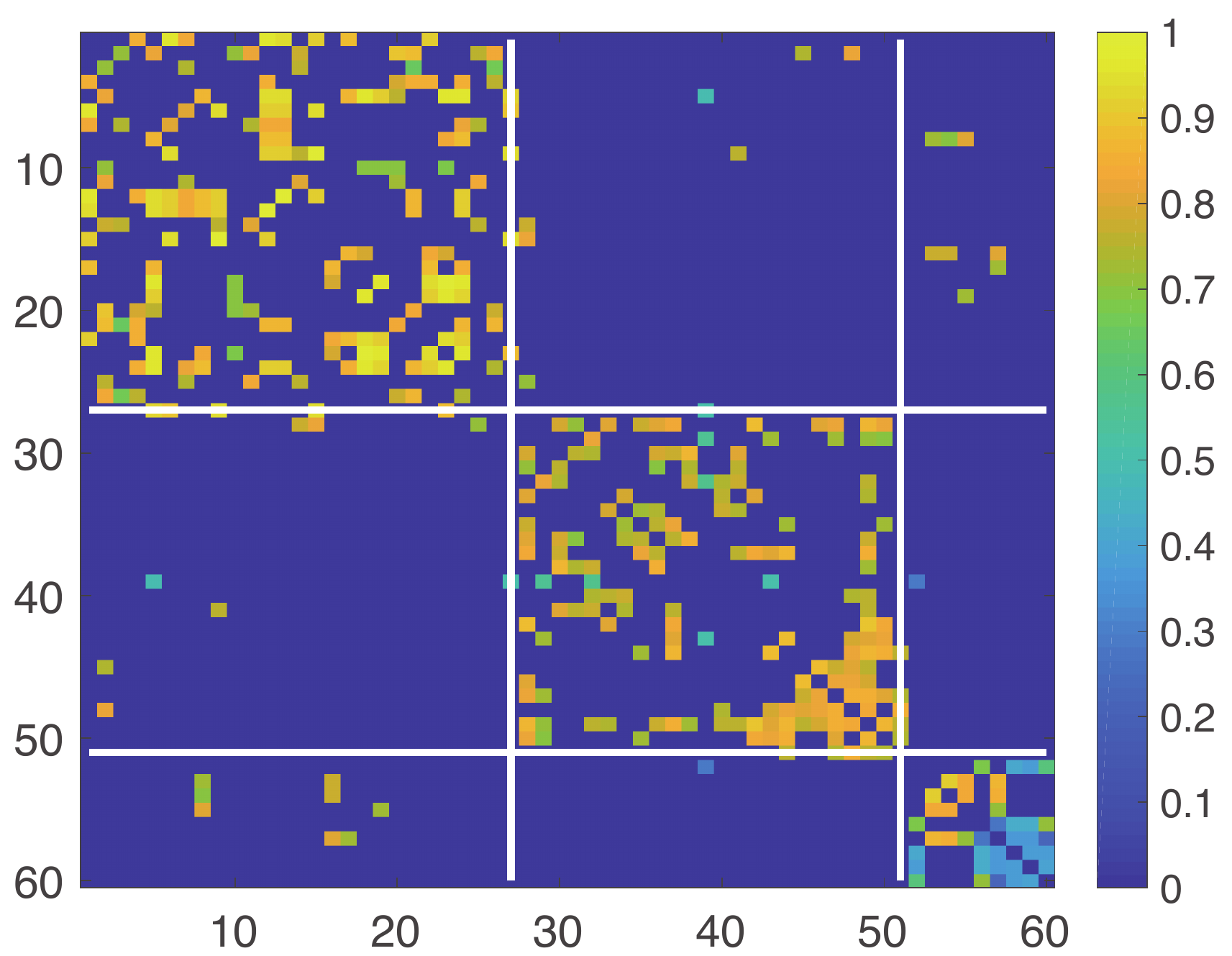}}
  \subfigure[SC-Yale]
    {\includegraphics[width=2.3in]{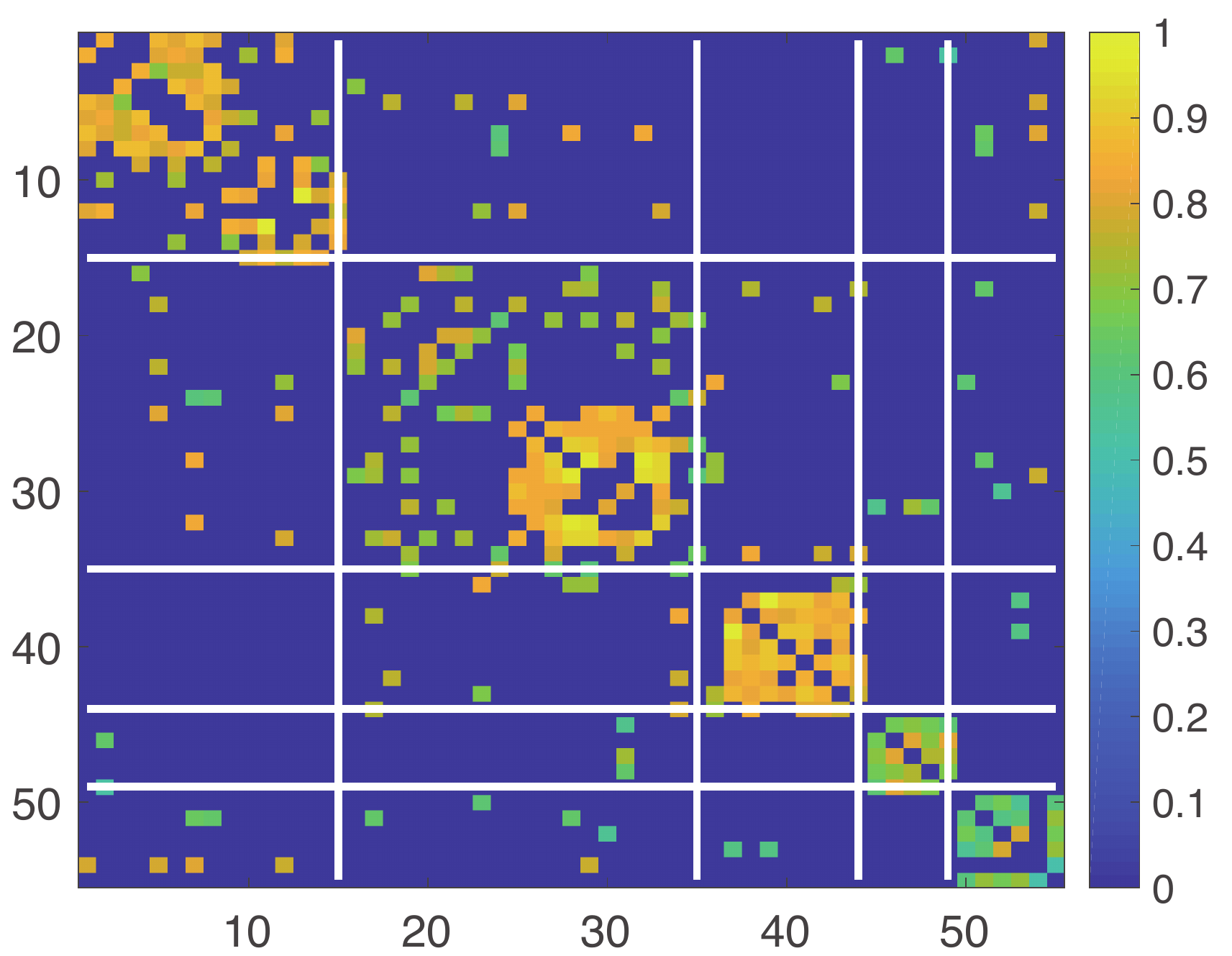}}
  \subfigure[PPC-Yale]
    {\includegraphics[width=2.3in]{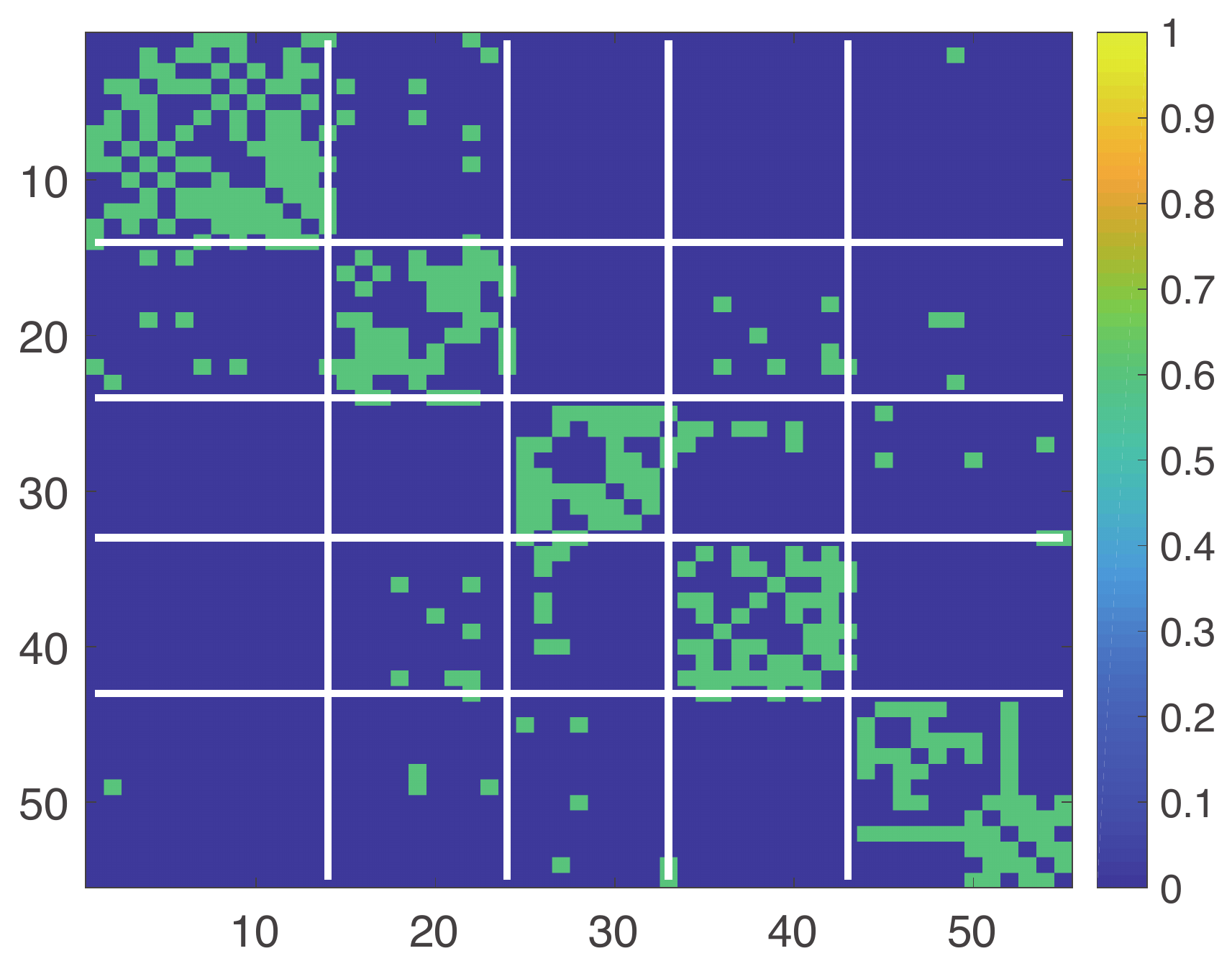}}
  \subfigure[IPS2-Yale]
    {\includegraphics[width=2.3in]{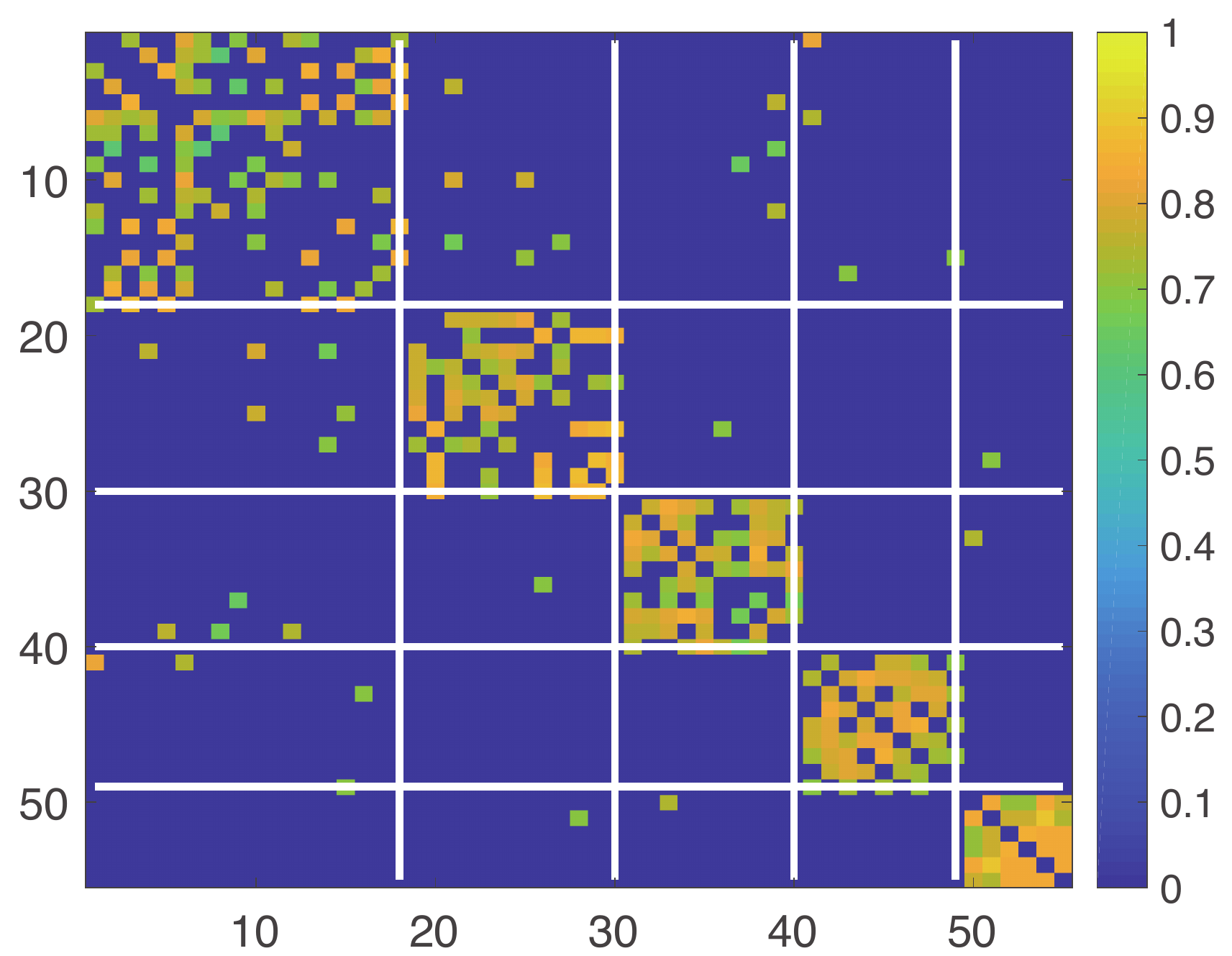}}
  \caption{The similarity heatmaps of SC, PPC and IPS2 over BBCSports and Yale datasets. To illustrate, SSC-BBCSports represents the heatmap of the pairwise similarity of SC over BBCSport, and the others are nominated in a similar way. These heatmaps demonstrate the learned high order similarity by IPS2 could relieve the influence of irrelevant features, and thereby enhance the robustness of the algorithm when facing high-dimensional datasets.}\label{fig:heatmap}
\end{figure*}

First, we observe that the proposed IPS2 based on the fused similarity uniformly outperforms SC and SSC, which solely depend on the pairwise similarity for clustering. For example, IPS2 exceeds SSC by 39.4$\%$, 58.1$\%$, 30.8$\%$ and 17.2$\%$ in term of ACC on all four datasets. These results demonstrate the effectiveness of the high order information, and further, the combination of the learned high order and pairwise similarity are way more powerful for achieving robust and accurate clustering performance. In addition, one may notice that SSC have significantly poor performance on SCADI. This result may come from the fact that SCADI violates the underlying assumption that SSC depends on.

Second, we note that the performance of CAVE and HGC are inferior to SC on SCADI, BBCSports and Yale, reflecting that high order relationships explored by averaging all samples within constructed hyperedges are inconducive for improving clustering performance on imbalanced datasets. By contrast, the proposed PPC performs comparably and even slightly better than SC, and is superior to the two hypergraph clustering algorithms. For example, PPC exceeds HGC by 19.2$\%$, 33.4$\%$, 13.4$\%$ and 7.3$\%$ in terms of ACC on four datasets. These results verify the effectiveness of the high order information explored by PPC and suggest that it has potential in enhancing the robustness of IPS2.

Third, our proposed IPS2 significantly outperforms the pairwise or high order similarity-based models. For example, IPS2 exceeds the best one of them by 5$\%$, 1.7$\%$, 18.3$\%$ and 12.7$\%$ in terms of ACC on Soybean, SCADI, BBCSport and Yale, respectively. These results indicate that the high order similarity indeed captures the evident structure of data that can be the strong complementary information for the pairwise similarity, leading to the superior performance.

To further validate the effectiveness of IPS2, we plotted the heatmaps of similarities produced by SC, PPC and IPS2 in Fig.~\ref{fig:heatmap}. Heatmaps of three columns refer to the pairwise similarity, the high order similarity and the fused similarity, respectively. In the first column, we notice that the pairwise similarities across two datasets contain loose clusters and exhibit many misclassified points within the off-diagonal areas. These results verify the limitation of the pairwise similarity. By contrast, the high order similarities have noticeable improvements. For example, the high order similarity on BBCSports possesses less misclassified samples within the off-diagonal regions and a more clear boundary along the diagonal, which the pairwise similarity fails to obtain. With the complementary information of the high order similarity, it is no wonder to see an evident and robust cluster distribution of fused similarities in IPS2. These results are consistent with observations in Tab.~\ref{results}.

\subsection{The Robustness of IPS2 on Under-sampled Datasets}

Since IPS2 fuses high order and pairwise information, it can effectively show the neighborhood structure of dataset that classical similarity-based methods fail to do. To verify this, we further investigated the performance of IPS2 over under-sampled datasets. The under-sample is either caused by the high-dimensional-low-sample-size or the class-imbalance. Therefore, we categorized the under-sampling problem into the foregoing two types. Accordingly, two synthetic datasets, named by USdata1 and USdata2, were constructed.

\begin{itemize}
 \item\textbf{USdata1} The first synthetic dataset consists of three clusters, and each one has 20 samples. They are drawn from i.i.d. Normal distribution with an equal standard deviation of 0.5, and different mean values of 0.1, 0.5 and 1, respectively. The dimensionality of the feature varies from 60 to 1,860, and thus the under-sample imbalance is gradually becoming severe. The dataset is further contaminated by white noise with zero mean value and 0.5 standard deviation.

\item\textbf{USdata2} The second synthetic dataset is also sampled from the same three clusters. Differently, The feature dimension is fixed to be 2,360. We also fixed the sample size in cluster 1 and 2 to be 20, while varying the sample size from 10 to 235 in cluster 3. The resulted dataset thus contains many more samples from cluster 3 than others, yielding a class-imbalance. The second dataset is both high-dimensional-low-sample-size and class-imbalance. It is rather challenging to cluster on this dataset.
\end{itemize}

\begin{figure*}[!htb]
  \centering
  \subfigure[USdata1]
   {\includegraphics[width=3.62in]{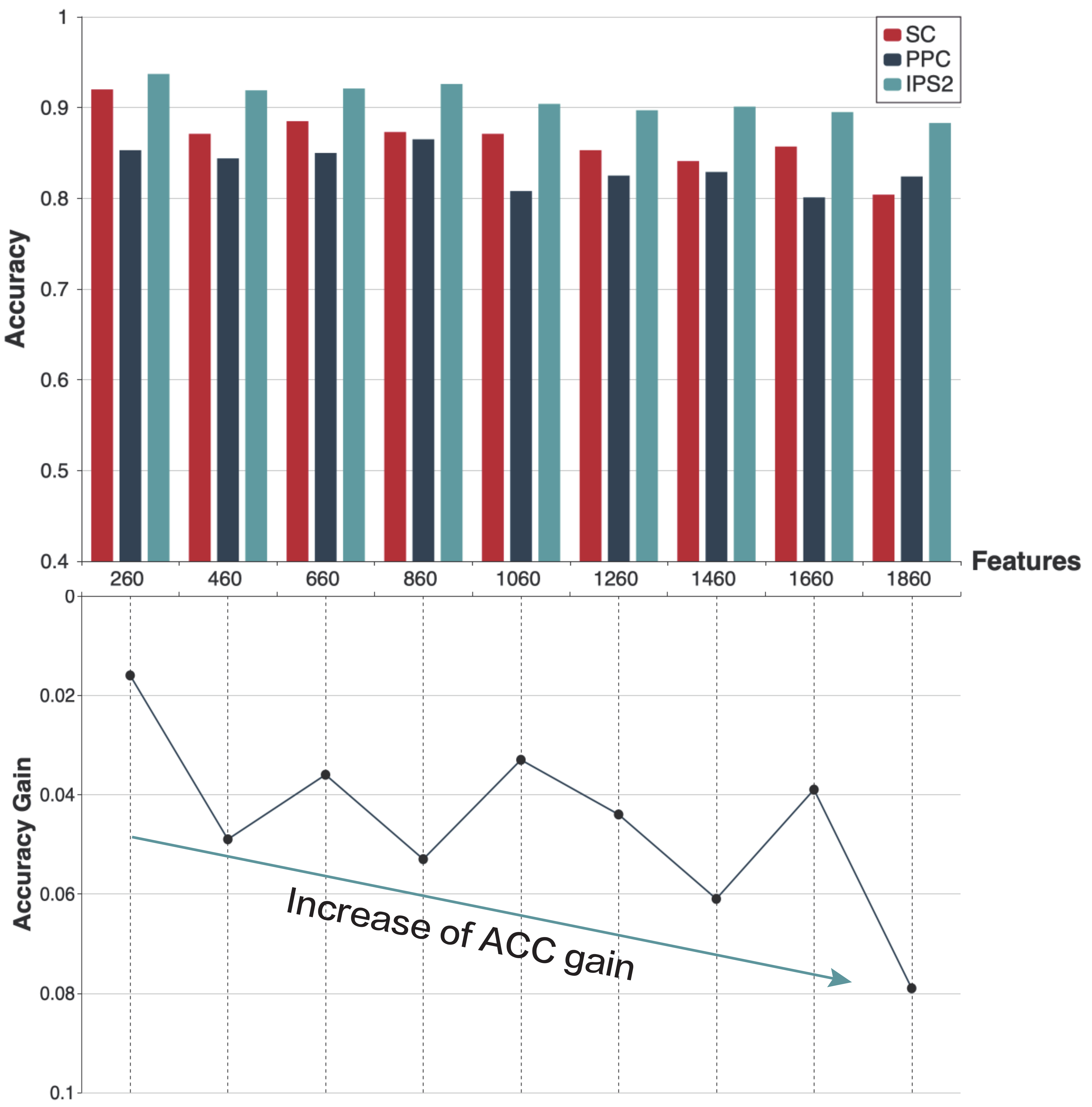}}
  \subfigure[USdata2]
   {\includegraphics[width=3.5in]{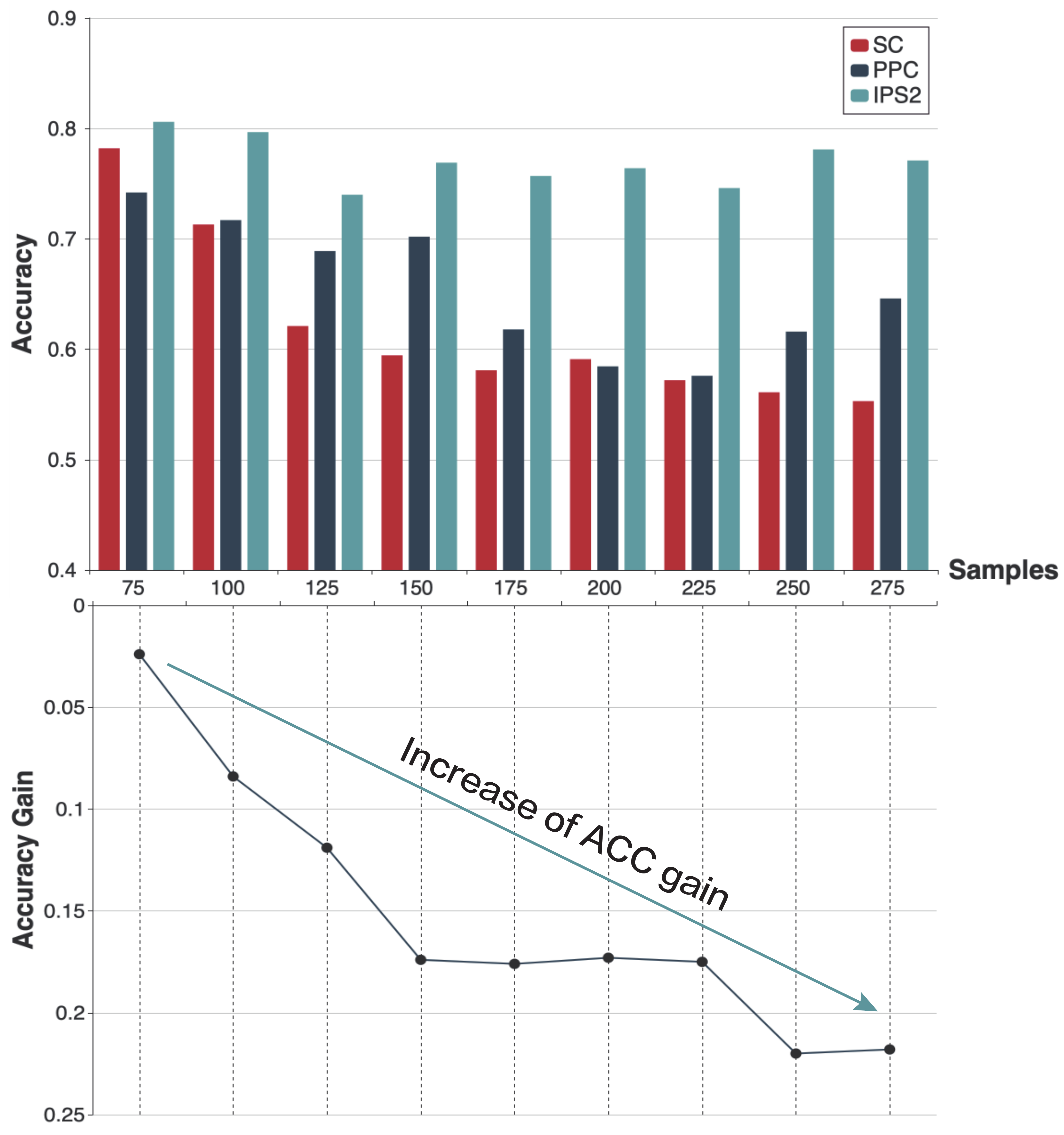}}
 \caption{Example for validating the robustness of IPS2 over under-sampled datasets. (a): accuracy of SC, PPC, and IPS2 on the USdata1 with the varying feature dimension from 260 to 1860; (b): accuracy of SC, PPC, IPS2 over the USdata2 with the varying sample size from 75 to 275. The lines denote the accuracy gain of IPS2 over SC. The superior performance of IPS2 indicates that incorporating the high order similarity is beneficial to improve the accuracy and robustness of the pairwise similarity even on the datasets with serious imbalance.}\label{fig:diff_fea_samples}
\end{figure*}

We applied SC, PPC and IPS2 on the two datasets, and the results are presented in Suppl. Tab. 1 and Suppl. Tab. 2~(\textbf{Suppl. Section 4}). The accuracy are visualized in Fig.~\ref{fig:diff_fea_samples}. Besides, we plotted the accuracy gain of IPS2 with respect to SC to further verify the effectiveness of the high order similarity.

In the two under-sampled datasets, the curse of dimension and class-imbalance make one sample similar to the others in different clusters, leading to low accuracy in clustering. Compared with SC and PPC, which solely depend on the pairwise or high order relationships, the fused similarity in IPS2 can capture an apparent structure of datasets. Therefore, our IPS2 outperforms uniformly over SC and PPC throughout all the experiments, even on the datasets with severe imbalance. The advantage is more evident when observing the accuracy gain over SC. In Fig.~\ref{fig:diff_fea_samples}, the gain increases as the imbalance rate increases in two datasets and yields its peak at 8$\%$ and 21.8$\%$ on the dataset with the highest imbalance rate. This result implies that the extracted high order similarity from the tensor similarity provide supplementary information, and combining it with the pairwise similarity enhances the robustness over the imbalance influence. Besides, it is worth to noted that PPC outperformed SC when the sample size of USdata2 increases from 100 to 275. This results confirmed that under-sampling imbalance breaks down the neighborhood structure of the pairwise affinity, and leads to the poor clustering performance when the imbalance becomes more and more severe. By comparison, the high order relationships in PPC indeed relieve such a problem. Therefore, IPS2 performs steadily with a slight decrease in accuracy by combining the high order and pairwise relationships.

\subsection{The Stability of IPS2 on Noisy Datasets}

Except for imbalance, real-world datasets are likely to be polluted by various types of noise in applications. Thus, noise immunity is of significance for clustering models. To evaluate the robustness of IPS2 to noise, we tested its performance with SC and PPC under four types of noise. The synthetic USdata1 was borrowed for testing. It was contaminated by the following types of noises. The parameter settings of these noise distributions are as follows.

\begin{figure}
\centering
 \includegraphics[width=3.5in]{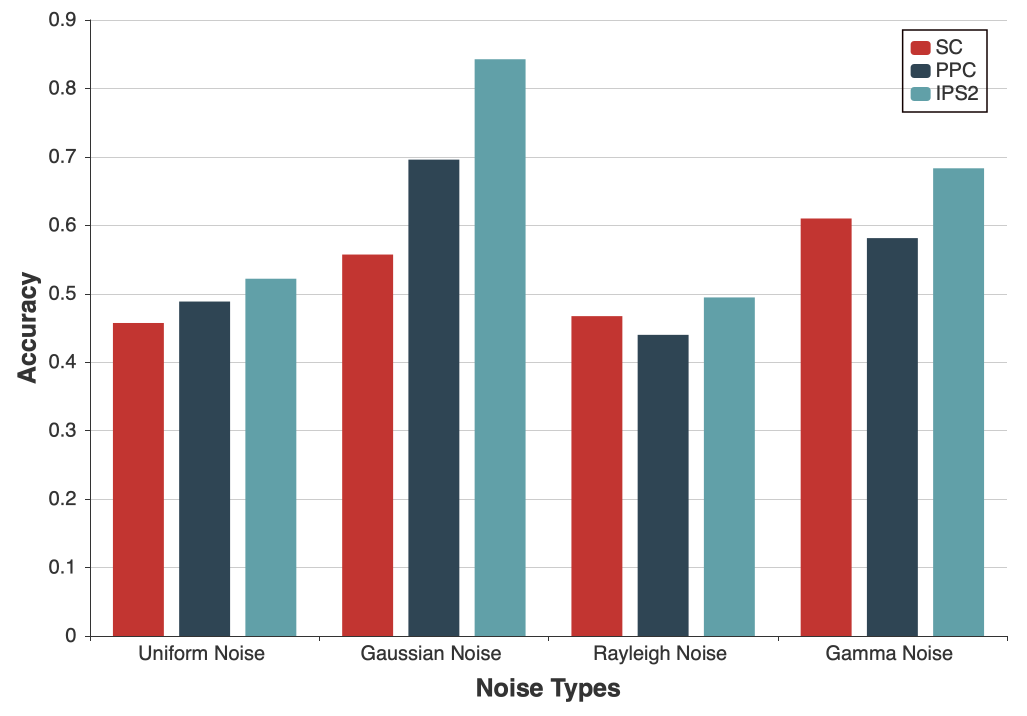}
 \caption{The performance of SC, PPC and IPS2 over synthetic dataset with different types of noise indicates the noise immunity of IPS2.}\label{fig:different_noise}
\end{figure}
\begin{itemize}
 \item \textbf{Uniform Noise} are drawn from Uniform distribution. We used the single uniform distribution, which returned random numbers in the interval (0,1).

 \item \textbf{Gaussian Noise} are drawn from Gaussian distribution $p(x)=\frac{1}{\sqrt{2 \pi} \sigma} \mathrm{e}^{-(x-\mu)^{2} / 2 \sigma^{2}}$. We set the mean $\mu=0$ and the covariance $\sigma = 0.5$.

 \item \textbf{Rayleigh Noise} are drawn from Rayleigh distribution $y=f(x | b)=\frac{x}{b^{2}} e^{\left(\frac{-x^{2}}{2 b^{2}}\right)}$, which is a special case of the Weibull distribution. The scale parameter $B = 0.5$.

 \item \textbf{Gamma Noise} are drawn from Gamma distribution. Gamma probability distribution function is $y=f(x | a, b)=\frac{1}{b^{a} \Gamma(a)} x^{a-1} e^{\frac{-x}{b}}$ where $\Gamma(.)$ is Gamma function. We set two scale parameter $A = 5$ and $B = 10$.
\end{itemize}

After USdata1 was corrupted by the noise, we ran SC, PPC and IPS2 on them. The experiment on each noise contamination was repeated 50 times, and the averaged results were recorded in Suppl. Tab. 3~(\textbf{Suppl. Section 5}). For visual comparison, the accuracy of three methods was plotted in Fig.~\ref{fig:different_noise}. One can verify that IPS2 uniformly outperforms the others on USdata1 through four kinds of noise. It achieves the best performance when the samples were corrupted by Gaussian noise.

We deem that IPS2 achieves such superior performance because the indecomposable tensor similarity is less affected by the Gaussian noise. Therefore, the obtained high order similarity reflects the accurate clustering information. Accordingly, both the PPC, whose performance solely depends on the high-order similarity and IPS2, achieve uniformly superior results. To further verify this, we tested the robustness of three methods over Gaussian noise contamination in USdata1 with different noise level. We varied the standard deviation of Gaussian noise from a small value of 0.2 to a large value of 0.8. The performance of three methods under various noise levels is presented in Suppl. Tab. 4~(\textbf{Suppl. Section 5}) and their accuracy is drawn in Fig.~\ref{fig:noise_level}. We observe that IPS2 always performs superior over the others. With the increase of noise level, the accuracy of all methods decreases rapidly. When the noise level is small, the IPS2 performed equally well to SC with the accuracy of 0.906 and 0.913, respectively. However, when the noise level increases, the accuracy of SC dramatically decreases. The accuracy of PPC decreases rapidly in the same manner. However, the accuracy of IPS2 remains relatively stable by unifying the merits of both the pairwise and high order similarities. Even when the noise level is as high as 0.8, IPS2 achieves a high accuracy of 0.760. In comparison, both PPC and SC yield a lower accuracy of 0.529 and 0.496, respectively.

\begin{figure}
  \centering
  \includegraphics[width=3.5in]{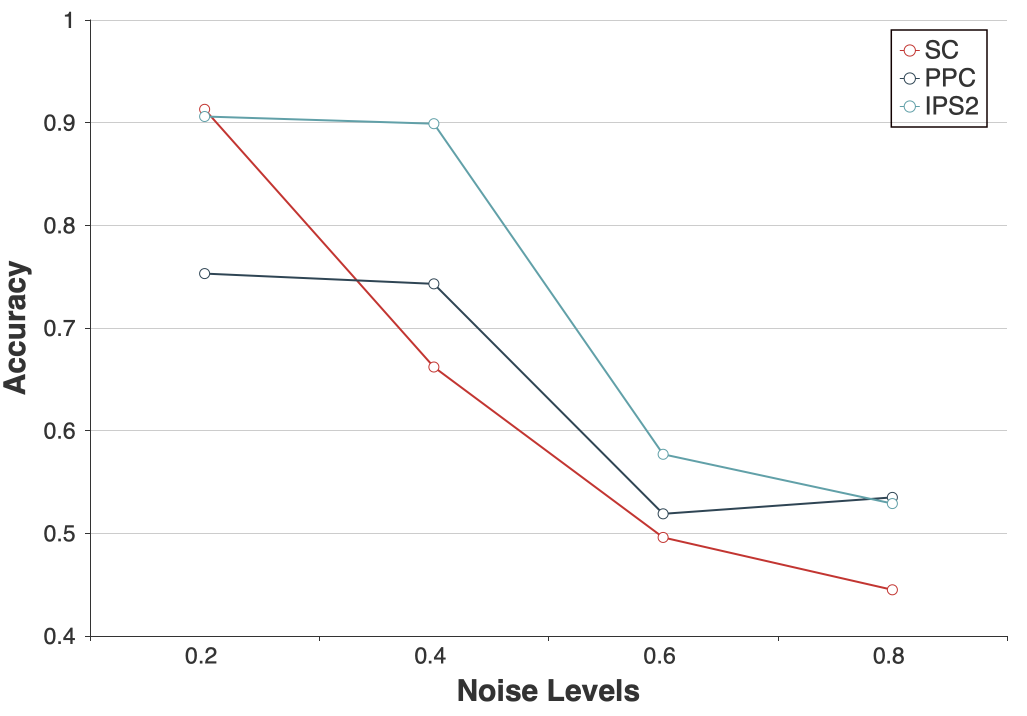}
   \caption{The performance of SC, PPC and IPS2 on synthetic datasets with different levels of Gaussian noise demonstrates that the neighborhood information captured by the high order similarity is robust to the noise contamination, even when it becomes highly destructive.}\label{fig:noise_level}
\end{figure}

\subsection{Computational Time Comparison and Parameter Selection}\label{sub:computation}

 The IPS2 consists of two major computational components which need an intensive cost. One is the construction of the tensor similarity, and the other is computing the eigenvectors of the Laplacian matrix for the unfolded tensor. For the former, we compute the similarity for pairs only by their the $k$-nearest neighbors. Given $m$ samples, this scheme reduces time complexity from $\mathcal{O}(m^4)$ to $\mathcal{O}(mk^4)$. We can further accelerate the speed by $\mathcal{O}(k^4)$ if using parallel programming in $m$ branches. For the latter, the tensor similarity is sparse, and thus the corresponding unfolded Laplacian matrix is sparse, too. For the sparse eigenvalue problem $\hat{\vT}\vx = \lambda \vx$, the computational time complexity is $\mathcal{O}(pm^2)$ by using Arnoldi algorithm, where $p$ is the mean number of nonzero elements of all columns in matrix $\hat{\vT}$, defined as $p=\frac{\#~nonzero~elements~in~\hat{\vT}}{\#~columns}$. In our scheme $p=\frac{mk^4}{m^2}=\frac{k^4}{m}$. In summary, the total computational cost is $\mathcal{O}(k^4)+\mathcal{O}(\frac{k^4}{m}m^2)=\mathcal{O}(mk^4)$.

\begin{table*}[!t]
 \centering
 \caption{Computational cost by six methods~(seconds) }\label{tab:time}
 \renewcommand{\arraystretch}{1.5}
  \begin{tabular}{ccccccc}
   \multicolumn{7}{c}{}\\
  \hline
  Datasets &SC  & HGC & AVER & SSC & PPC &IPS2\\
  \hline
     Soybean     &1.54 & 4.47 &3.54 &803.11 &15.46 &16.52\\
     SCADI      & 2.18 & 14.73 &5.08 &1267.10 &57.93 &58.91\\
     BBCSports  & 3.91 & 7.53 &5.65 &1074.36 &30.07 &31.44\\
     Yale        & 8.20 &12.58 &8.45 &1992.46 &30.14 &30.96\\
       \hline
     Time cost &$\mathcal{O}(m^2)$ &$\mathcal{O}(cm^2)$ &$\mathcal{O}(m^2)$ &- &$\mathcal{O}(k^4)$ &$\mathcal{O}(mk^4)$\\
  \hline
  \end{tabular}
\end{table*}

A small value of $k$ dramatically reduce the complexity, yet, it may be insufficient to capture the complementary information among pairs. We conducted a grid search for the parameter within interval $[5,15]$ on USdata1. The accuracy with respect to the value $k$ is plotted in Fig.~\ref{increase features}. One may notice that the accuracy goes up with the value of $k$ increasing. Around the value of 7, the trend remains relatively steady. Therefore, in this paper, the value of $k$ is pruned within range of [7,10] to balance the accuracy and complexity. As a rule of thumb, the value $k$ is set to be 10.

\begin{figure}
  \begin{minipage}[b]{.25\textwidth}
  \centering
  \includegraphics[width=3.5in]{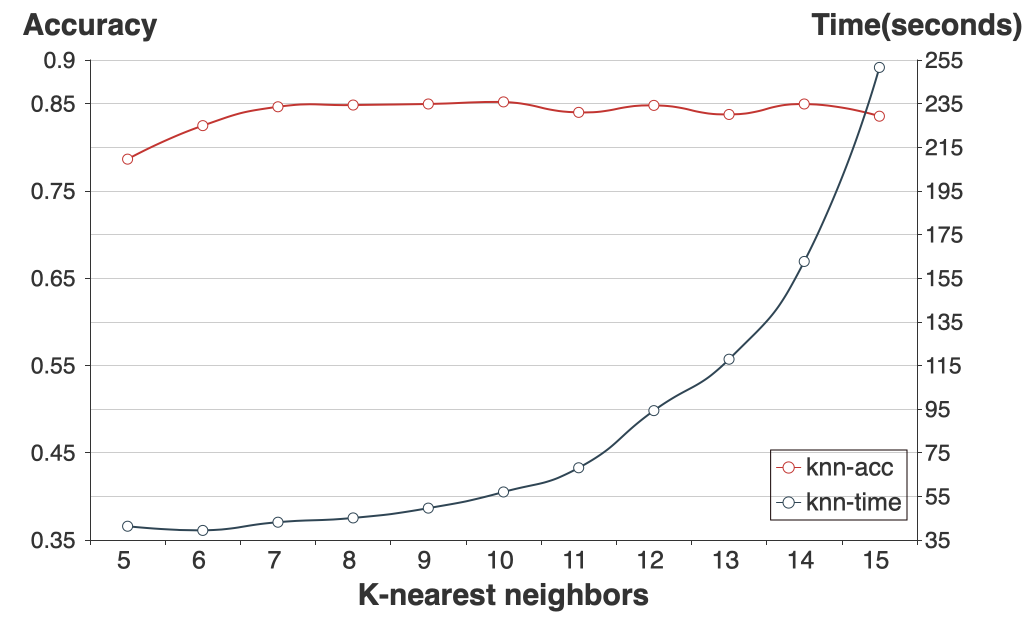}
  \end{minipage}
  \vskip -.3cm
 \caption{The time complexity and accuracy based on different $k$ in K-nearest-neighbors-search suggests that choosing $k$ in a suitable range could balance the complexity and accuracy by IPS2. }\label{increase features}
\end{figure}

To illustrate the comparison of the computational cost, we summarize the experimental computational time of all methods on four real datasets in Tab.~\ref{tab:time}. The time cost in Tab.~\ref{tab:time} is the averaged time cost of 50 runs. The value of $k$ for the IPS2 is set to be 10. Among the six methods, both the HGC and AVER need to construct a hypergraph by calculating the pairwise similarity, and therefore their computational cost is $O(m^2)$. The cost of HGC is slightly higher than that of AVER since the former is also dependent on the cluster number $c$. Except for SSC, which needs many iterations to converge, other methods cost less than 60 seconds. Although the theoretical complexity of IPS2 is higher than SC, HGC and CAVE, yet, it is acceptable in a real application by considering its superior performance. For example, SC is 15 times faster than IPS2, while its accuracy is 15 $\%$ less than that of IPS2. Moreover, one can further reduce the complexity of IPS2 in the step of constructing the tensor similarity by using parallel computation. Thus, the computational cost of IPS2 can reduced to an acceptable degree.

\section{Conclusions}\label{conclusion}

The performance of the classic clustering methods pins on an accurate pairwise similarity measurement. However, the measurement is notoriously known for being vulnerable to noise contaminations and failing to capture the neighborhood structure for the dataset with a small sample size and large feature dimensionality. The current work, serving as a proof-of-concept study, demonstrates that utilizing affinities among multiple samples are able to boost the clustering performance of conventional clustering algorithms. Accordingly, we define a tensor similarity for pairs which aims to provide the complementary information that a pairwise similarity missed. There are two types of tensor similarity: decomposable and indecomposable one. In the decomposable case, we connect the tensor similarity with a pairwise similarity, proving that the unfolding of the tensor similarity is equivalent to the Kronecker product of a pairwise similarity. For the indecomposable tensor similarity, empirical evidence shows that it captures the evident structure of data. Benefiting from this, we extract the high order similarity from the indecomposable tensor, which can serve as the complementary information for pairwise relationships. Finally, the extracted high order and pairwise similarity are fused for achieving accurate clustering performance. Extensive experiments demonstrate that the combination of the pairwise and high order similarities effectively achieves the accurate and robust cluster performance. The proposed IPS2 outperforms all baseline methods, especially in handling under-sampled and noisy dataset.

\ifCLASSOPTIONcompsoc
  \section*{Acknowledgments}
\else
  \section*{Acknowledgment}
\fi

This work was partially supported by the National Natural Science Foundation of China (61472145,61771007), Science and Technology Planning Project of Guangdong Province (2016B010127003), Health \& Medical Collaborative Innovation Project of Guangzhou City (201803010021).

\ifCLASSOPTIONcaptionsoff
  \newpage
\fi



%\bibliographystyle{IEEEtran}

\bibliographystyle{IEEEtran}
\bibliography{IPS2}

\end{document}